%% file: colm2026_conference.tex
\documentclass{article} 
\usepackage[final]{colm2026_conference}

\usepackage{microtype}
\usepackage{hyperref}
\usepackage{url}
\usepackage{booktabs}
\usepackage{multirow}
\usepackage{graphicx}

\usepackage{lineno}
\usepackage{colortbl}
\definecolor{darkblue}{rgb}{0, 0, 0.5}
\hypersetup{colorlinks=true, citecolor=darkblue, linkcolor=darkblue, urlcolor=darkblue}

\usepackage{algorithm}
\usepackage{algpseudocode}
\usepackage{amsmath,amssymb,amsfonts}
\usepackage{amsthm}
\usepackage{hyperref}
\usepackage{cleveref}
\newtheorem{theorem}{Theorem}
\newtheorem{proposition}[theorem]{Proposition}

\newtheorem{lemma}[theorem]{Lemma}
\newtheorem{definition}[theorem]{Definition}

\newtheorem{remark}{Remark}

\usepackage{wrapfig}
\usepackage{xcolor}
\usepackage[table]{xcolor}

\crefname{theorem}{Theorem}{Theorems}
\crefname{proposition}{Proposition}{Propositions}
\crefname{corollary}{Corollary}{Corollaries}
\crefname{assumption}{Assumption}{Assumptions}
\crefname{definition}{Definition}{Definitions}
\crefname{algorithm}{Algorithm}{Algorithms}
\numberwithin{theorem}{section}
\numberwithin{assumption}{section}

\newcommand{\repair}{\textsc{Repair}}
\usepackage{subcaption}

\usepackage{fontawesome5}
\usepackage{hyperref}
\input{math_commands}

\title{\textsc{Beyond Logit Adjustment:} A Residual Decomposition Framework for Long-Tailed Reranking}
\author{Zhanliang Wang$^{1,2}$, \;
Hongzhuo Chen$^{1,2}$, \;
Quan Minh Nguyen$^{1,2}$, \\
\textbf{Mian Umair Ahsan$^{2}$, \;
Kai Wang$^{1,2}$}\thanks{Corresponding author. Email: \texttt{wangk@chop.edu}} \\[4pt]
{\small $^1$ University of Pennsylvania, Philadelphia, PA, USA} \\
{\small $^2$ Children's Hospital of Philadelphia, Philadelphia, PA, USA}
}

\begin{document}

\ifcolmsubmission
\linenumbers
\fi

\maketitle
\begin{abstract}
Long-tailed classification, where a small number of frequent classes dominate many rare ones, remains challenging because models systematically favor frequent classes at inference time. Existing post-hoc methods such as logit adjustment address this by adding a fixed classwise offset to the base-model logits. However, the correction required to restore the relative ranking of two classes need not be constant across inputs, and a fixed offset cannot adapt to such variation. We study this problem through Bayes-optimal reranking on a base-model top-$k$ shortlist. The gap between the optimal score and the base score, the \emph{residual correction}, decomposes into a \emph{classwise} component that is constant within each class, and a \emph{pairwise} component that depends on the input and competing labels. When the residual is purely classwise, a fixed offset suffices to recover the Bayes-optimal ordering. We further show that when the same label pair induces incompatible ordering constraints across contexts, no fixed offset can achieve this recovery. This decomposition leads to testable predictions regarding when pairwise correction can improve performance and when cannot. We develop \repair{} (\textbf{Re}ranking via \textbf{Pair}wise residual correction), a lightweight post-hoc reranker that combines a shrinkage-stabilized classwise term with a linear pairwise term driven by competition features on the shortlist. Experiments on nine benchmarks confirm that the decomposition explains where pairwise correction helps and where classwise correction alone suffices. These span text classification, visual recognition, and multimodal rare-disease diagnosis.

\vspace{0.8em}
\noindent\textbf{Code:} \faGithub\ \url{https://github.com/WGLab/REPAIR}
\end{abstract}

\input{introduction}
\input{related_work}
\input{preliminaries}
\input{framework}
\input{algorithm}

\input{experiment}

\input{conclusion}
\section*{Acknowledgments}
This study is supported by NIH grant \emph{OD037960} and the Children's Hospital of Philadelphia (CHOP) Research Institute. We also gratefully acknowledge the Penn Advanced Research Computing Center (PARCC) for providing computational resources and support.

\section*{Ethics Statement}
This work applies post-hoc reranking to rare-disease diagnosis, a high-stakes clinical setting. REPAIR is designed as a decision-support tool that reorders a candidate shortlist; it does not make final diagnostic decisions. This study is approved by the institutional review board at the Children's Hospital of Philadelphia (CHOP) under protocol \emph{IRB 18-015712}. All datasets used in this work are publicly available.

\section*{Use of Generative AI} We used large language models (Claude Opus 4.6) to polish grammar and improve the clarity of our writing, to organize the structure and logical flow of the paper, and to check and refine the presentation of mathematical proofs. All scientific contributions, experimental design, and final content were produced and verified by the authors.
\bibliography{colm2026_conference}
\bibliographystyle{colm2026_conference}
\input{appendix}

\end{document}

%% file: math_commands.tex
\newcommand{\argmax}{\mathop{\mathrm{arg\,max}}}



%% file: introduction.tex
\section{Introduction}
\label{sec:intro}

Long-tailed classification, where a few common classes dominate many rare ones, remains a persistent challenge because models systematically favor frequent classes at prediction time, pushing rare labels down the ranking even when they are plausible.
A common post-hoc correction method is logit adjustment~\citep{menon2021longtail}, which adds a fixed class-dependent offset to the base-model scores. This is effective when rare classes are globally under-ranked due to bias introduced by training frequency, but long-tailed ranking errors are more complicated than a single classwise offset can express.

Consider an animal classifier trained on far more images of \emph{cat} than \emph{leopard} (Figure~\ref{fig:overview}). 
On a clear photograph showing a leopard's rosette pattern, the model assigns $56\%$ to cat and $44\%$ to leopard. On a distant, partially occluded shot of a leopard, the model gives $70\%$ to cat and only $30\% $to leopard. On a third image that genuinely depicts a cat, the model gives $67\%$ to cat and $33\%$ to leopard. No single fixed per-class offset works for all three: a value small enough to preserve the correct ranking on the third image cannot close the gap on the second, while a value large enough to correct the second inevitably flips the third. 
The root cause is data imbalance: with few training images of leopard, the base model's scores for this class are not only biased but highly variable across inputs, and this variability is exactly what a fixed offset cannot absorb.



\begin{figure}[t]
\centering
\includegraphics[
  width=1.0\textwidth,
  trim=0pt 0pt 0pt 0pt,  
  clip
]{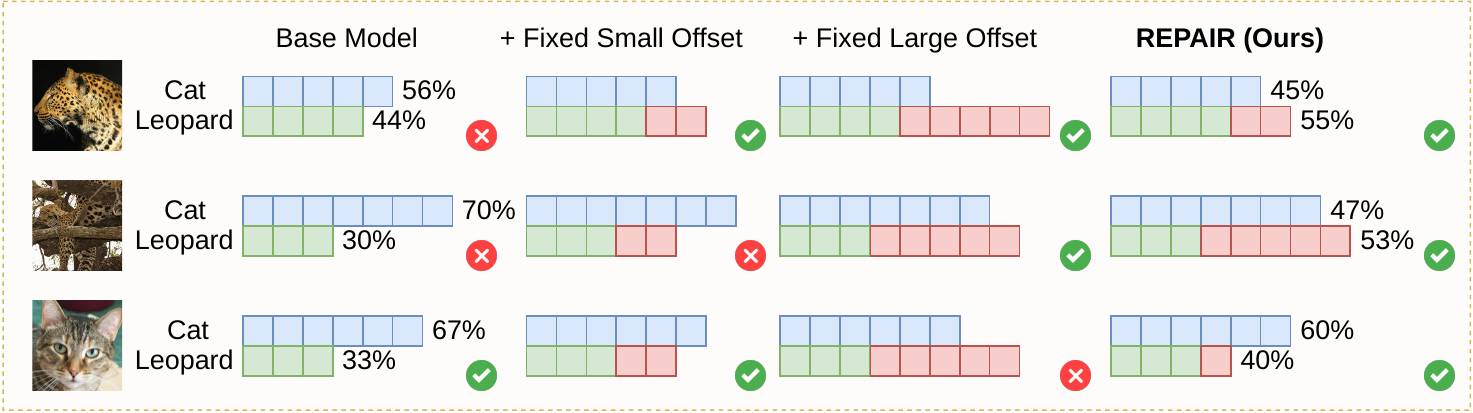}
\vspace{-8pt}  
\caption{Motivating example (shortlist size $2$ for clarity). Each row is a different input; 
bars show scores and percentages indicate softmax probabilities. A fixed small offset corrects the easy case (row~1) but under-corrects the hard case (row~2); a fixed large offset corrects both but flips the already-correct case (row~3). \repair{} adapts the correction magnitude per input, producing the correct ranking in all three cases.}
\vspace{-10pt}  
\label{fig:overview}
\end{figure}

We formalize this by studying Bayes-optimal reranking on a base-model shortlist. The gap between the optimal score and the base score, the \emph{residual correction}, decomposes into a \emph{classwise} component, which depends only on the label identity; and a pairwise component, which varies with the input and the competing labels; Section~\ref{sec:framework} characterizes precisely when the classwise component alone suffices and when it does not. Real datasets fall on a spectrum
between these extremes, and the degree of pairwise competition determines how much a
context-dependent correction can help. This raises a practical question: \emph{how should a reranker adapt its correction to the specific input and the competing labels on the shortlist?}
Our key insight is that the competition structure within the shortlist already encodes this information: how closely the base model scores the true label relative to each rival reflects how hard the current input is, and this relative strength varies across inputs, which is exactly the variation that a fixed offset misses.

Based on this analysis, we develop \textsc{Repair} (Reranking via Pairwise residual correction), a lightweight post-hoc reranker that combines a shrinkage-stabilized classwise term, which alone improves over logit adjustment by learning offsets on the shortlist, with a linear pairwise term driven by competition features. Both components are fitted jointly on covered held-out calibration examples; the base model is never modified.
We validate on nine long-tailed benchmarks spanning text classification, visual 
recognition, and rare-disease diagnosis. Consistent with the decomposition, the 
pairwise term yields small gains where the classwise correction already suffices, 
as on the vision benchmarks (ImageNet-LT~\citep{imagenetlt}, 
Places-LT~\citep{imagenetlt}, iNaturalist~\citep{vanhorn2018inaturalist}) and the 
class-separable text task GoEmotions~\citep{demszky2020goemotions}. Gains are substantially larger where it does not, as on the non-class-separable text tasks 
(MASSIVE~\citep{fitzgerald2023massive}, HuffPost-LT~\citep{misra2022news}, 
Few-NERD~\citep{ding2021fewnerd}) and rare-disease diagnosis on 
GMDB~\citep{lesmann2024gmdb} and its out-of-distribution variant 
RareBench~\citep{chen2024rarebench}).

\paragraph{Contributions.}
\begin{enumerate}
\item We decompose the residual correction on a fixed shortlist into classwise and pairwise components, proving that when the same label pair induces incompatible ordering constraints across inputs, no fixed per-class offset can recover the Bayes ordering (Theorem~\ref{thm:pairwise-necessary}).
\item We derive \textsc{Repair}, a lightweight post-hoc reranker from this decomposition, combining a learned classwise correction with a linear pairwise term.
\item We validate the framework's predictions on nine benchmarks: gains are proportional to the degree of pairwise competition and concentrated where the theory predicts fixed offsets to be the most fragile.
\end{enumerate}

%% file: related_work.tex
\section{Related Work}
\label{sec:related}

\paragraph{Long-tailed distribution.}
Real world class distributions are rarely uniform. Typically a few head classes dominate many rare ones, forming a long-tailed distribution~\citep{yang2022survey,zhang2023survey}. Models trained on such datasets are systematically biased towards frequent classes, which usually results in poor performance on the tail~\citep{ren2020balanced, menon2021longtail}. Existing solutions can be grouped into two main categories based on when the correction is applied: training-time methods and post-hoc methods.


\paragraph{Post-hoc score correction.}
Post-hoc methods adjust the output of a fixed pretrained model at inference time. Logit adjustment methods add a fixed per-class offset to correct the distribution shift~\citep{ren2020balanced,menon2021longtail,disentangling}. Weight normalization methods rescale the classifier weight norms to remove the correlation between norm magnitude and class frequency~\citep{kang2020decoupling,kim2020adjusting}. The correction for each class remains fixed, which is a common limitation shared by all these methods. Temperature scaling~\citep{guo2017calibration,xiao2025restoring} is also post-hoc but rank-preserving,
targeting calibration rather than ordering. A broader review of training-time methods is provided in Appendix~\ref{app:add_related}.

\paragraph{Rare-disease diagnosis.}
Many rare genetic disorders present distinctive facial features, and CNN-based systems~\citep{DeepGestalt, hsieh2022gestaltmatcher} can prioritize candidate diagnoses from facial photos across hundreds of syndromes. Recent multimodal approaches further integrate clinical notes and demographic information via large language models~\citep{gestaltmml, CoT, mint}, yet the extreme class imbalance inherent in rare disease settings remains a core challenge across all these systems.

%% file: preliminaries.tex
\section{Preliminaries}
\label{sec:prelim}
 
We consider a multiclass classification problem with input $X$, label $Y \in [K] = \{1,\dots,K\}$, and long-tailed class prior $\pi_y = P(Y = y)$.
A trained base model produces scores $g(X) = (g_1(X),\dots,g_K(X)) \in \mathbb{R}^K$.
 
\paragraph{Shortlist and coverage.}
Fix an integer $k \ll K$ and let $S(X) \subseteq [K]$ denote the deterministic top-$k$ shortlist under $g(X)$, with deterministic tie-breaking.
A \emph{reranker} selects a label $\hat{y}(X) \in S(X)$.
Write $C := \{Y \in S(X)\}$ for the \emph{coverage event}; every reranker satisfies $P(\hat{y}(X) = Y) \leq P(C) = \text{Recall@}k$, so shortlist recall is a hard ceiling on reranking accuracy.
Since reranking can only affect covered examples, we condition all subsequent analysis and evaluation metrics on~$C$.
 
\paragraph{Shortlist posterior.}
Once the shortlist is fixed, the relevant object for reranking is the posterior restricted to the shortlisted labels:
\begin{equation}
\label{eq:shortlist-posterior}
p_S(y \mid x) := P(Y = y \mid X = x,\, C), \qquad y \in S.
\end{equation}


The corresponding log-posterior,
\begin{equation}
\label{eq:bayes-score}
s^\star_y(x, S) := \log p_S(y \mid x),
\end{equation}
is the Bayes-optimal shortlist score: a standard pointwise argument (Appendix~\ref{proof:bayes-shortlist}) shows that the Bayes-optimal reranker selects $\hat{y}^\star(x) \in \arg\max_{y \in S(x)} p_S(y \mid x)$.
Throughout, only relative values of $s^\star_y$ across $y \in S$ matter for ranking.
A generic classwise reranker modifies the base score as $g_y(x) + a_y$ for a fixed per-class offset~$a_y$; prior-based logit adjustment is the special choice $a_y = -\tau \log \pi_y$.
The next section discusses when this classwise form is sufficient and when it is not.

%% file: framework.tex
\section{A Residual Correction Framework for Long-Tailed Reranking}
\label{sec:framework}

Once a base model returns a fixed shortlist, the first question is what score a Bayes-optimal reranker should target.
The second is how the base score differs from this target.
That gap defines a \emph{residual correction}.
We then ask when this correction can be represented by a fixed classwise term, and when context-dependent competition makes a pairwise term necessary.

\subsection{Residual correction and pairwise gaps}
\label{sec:residual}

For each shortlisted label $y \in S$, define the \emph{residual correction}
\begin{equation}
\label{eq:residual}
\beta^\star_y(x,S) \;:=\; s^\star_y(x,S) \;-\; g_y(x).
\end{equation}
Because ranking depends only on score \emph{differences}, the relevant object is the \emph{pairwise gap}
\begin{equation}
\label{eq:pairwise-gap}
\Delta_{uv}(x,S) \;:=\; \beta^\star_u(x,S) - \beta^\star_v(x,S) \;=\; \bigl(s^\star_u - s^\star_v\bigr) - \bigl(g_u(x) - g_v(x)\bigr), \qquad u,v \in S.
\end{equation}
The quantity $\Delta_{uv}(x,S)$ is invariant to any shortlist-shared additive constant in $\beta^\star$ and captures exactly the correction difference needed to recover the Bayes ordering of the pair $(u,v)$.

\begin{definition}[Class-separable and non-class-separable regimes]
\label{def:class-separable}
The setting is \emph{class-separable} if there exists a vector $a \in \mathbb{R}^K$ such that for every covered $(x,S)$ and every $u,v \in S$,
\begin{equation}
\label{eq:class-sep-gap}
\Delta_{uv}(x,S) = a_u - a_v.
\end{equation}
Equivalently, there exists a shortlist-shared scalar $c(x,S)$ such that $\beta^\star_y(x,S) = a_y + c(x,S)$ for all $y \in S$.
If no such vector $a$ exists, the setting is \emph{non-class-separable}.
\end{definition}

Intuitively, class-separable means that the ranking bias of each class can be summarized by a single number.
Prior-based logit adjustment assumes exactly this structure: the correction depends only on the training frequency of the class, not on which rivals appear or how the input looks.

\begin{proposition}[Classwise correction as an exact special case]
\label{prop:classwise-special}
If the setting is class-separable, then the fixed-offset score $r_y(x,S) = g_y(x) + a_y$ induces the Bayes-optimal ordering on every covered example.
\end{proposition}

This positions logit adjustment precisely: it is exact when the residual correction is class-separable up to a shortlist-shared constant.
Proof in Appendix~\ref{proof:classwise-special}.
The next subsection shows why this condition fails in general.

\subsection{Why a pairwise correction is necessary}
\label{sec:pairwise-necessary}

For classes $u,v \in [K]$, define the \emph{base-score threshold}
\begin{equation}
\label{eq:threshold}
t(x;\,u,v) \;:=\; g_v(x) - g_u(x).
\end{equation}
A fixed classwise model ranks $u$ above $v$ if and only if $a_u - a_v > t(x;\,u,v)$.
Thus a single cut $a_u - a_v$ must satisfy the requirements of \emph{all} contexts in which the pair $(u,v)$ appears.

\begin{theorem}[Necessity of pairwise correction]
\label{thm:pairwise-necessary}
Suppose there exist two covered contexts $(x,S)$ and $(x',S')$ and two labels $u,v \in S \cap S'$ such that the Bayes rule ranks $u$ above $v$ at $x$, ranks $v$ above $u$ at $x'$, and
\begin{equation}
\label{eq:threshold-crossing}
t(x;\,u,v) \;\geq\; t(x';\,u,v).
\end{equation}
Then no fixed offset vector $a \in \mathbb{R}^K$ can agree with the Bayes ordering in both contexts.
\end{theorem}

Theorem~\ref{thm:pairwise-necessary} identifies the precise point at which fixed classwise correction breaks down.
Input-dependent pairwise variation alone does not force failure; the critical event is a \emph{threshold crossing} for the same pair across contexts.
We call such a pair \emph{contradictory}.
A setting may be non-class-separable without containing a contradictory pair, but once such a pair appears, no fixed classwise offset can agree with the Bayes ordering in all relevant contexts. In practice, a dataset can have large per-class score variance yet still be fully correctable by a fixed offset, as long as the variation does not reverse the ordering of any pair.
Proof in Appendix~\ref{proof:pairwise-necessary}. 

\paragraph{Threshold dispersion.}
Theorem~\ref{thm:pairwise-necessary} shows that threshold 
crossings drive the failure of fixed offsets. To quantify how 
prone a class is to such crossings, we define
\begin{equation}
\label{eq:threshold-dispersion}
D_y(S) \;:=\; \max_{j \in S \setminus \{y\}} 
\mathrm{std}\bigl(t(x;\,y,j)\bigr),
\end{equation}
where the standard deviation is computed across all covered 
contexts in which both $y$ and $j$ appear in the shortlist. 
Large $D_y$ means the true label faces at least one rival whose 
score gap fluctuates widely across inputs: some inputs are easy 
to distinguish, others are not, and no single fixed offset can 
handle both. In Section~\ref{sec:experiments}, we compute the 
mean threshold dispersion per rare class across all its covered 
test examples, then sort rare classes into five equal-sized 
quintiles (Q1--Q5, lowest to highest; we call these \emph{contradiction quintiles} since high $D_y$ is a necessary condition for threshold crossings) and show that the gain of 
\repair{} over classwise-only increases from Q1 to Q5.
\subsection{Measuring reranking gains}
\label{sec:rho}

Since all evaluation is conditioned on coverage, we normalize reranking gains by the recoverable gap within the shortlist:
\begin{equation}
\label{eq:gap-closed}
\rho_k(\hat{y}) \;:=\; \frac{\mathrm{Hit@1}(\hat{y}) \;-\; \mathrm{Hit@1}(\hat{y}_{\mathrm{base}})}{1 \;-\; \mathrm{Hit@1}(\hat{y}_{\mathrm{base}})},
\end{equation}
where $\mathrm{Hit@1}(\hat{y}) := P(\hat{y}(X) = Y \mid C)$.
The denominator is the fraction of covered examples the base model fails to rank first; $\rho_k = 1$ means the reranker corrects all such failures.
We use $\rho_k$ primarily in the synthetic validation below; for real-data experiments, we report Hit@1 and HFR directly since they are more interpretable across datasets with different coverage rates.
The equivalent unconditional form is given in Appendix~\ref{sec:rho-unconditional}.

\subsection{Synthetic validation}
\label{sec:synthetic}

To verify the predictions of Proposition~\ref{prop:classwise-special} and Theorem~\ref{thm:pairwise-necessary}, we construct two controlled settings: one where the residual is class-separable, and one where confuser pairs induce threshold crossings (setup in Appendix~\ref{sec:toy-details}).
Figure~\ref{fig:regime} confirms both predictions: in the class-separable setting, classwise reranking and REPAIR achieve the same $\rho_k$ ($= 0.17$); in the non-class-separable setting with contradictory pairs, REPAIR closes substantially more of the recoverable gap ($\rho_k = 0.42$ vs.\ $0.18$). 

\begin{figure}[!ht]
\centering
\includegraphics[width=\textwidth]{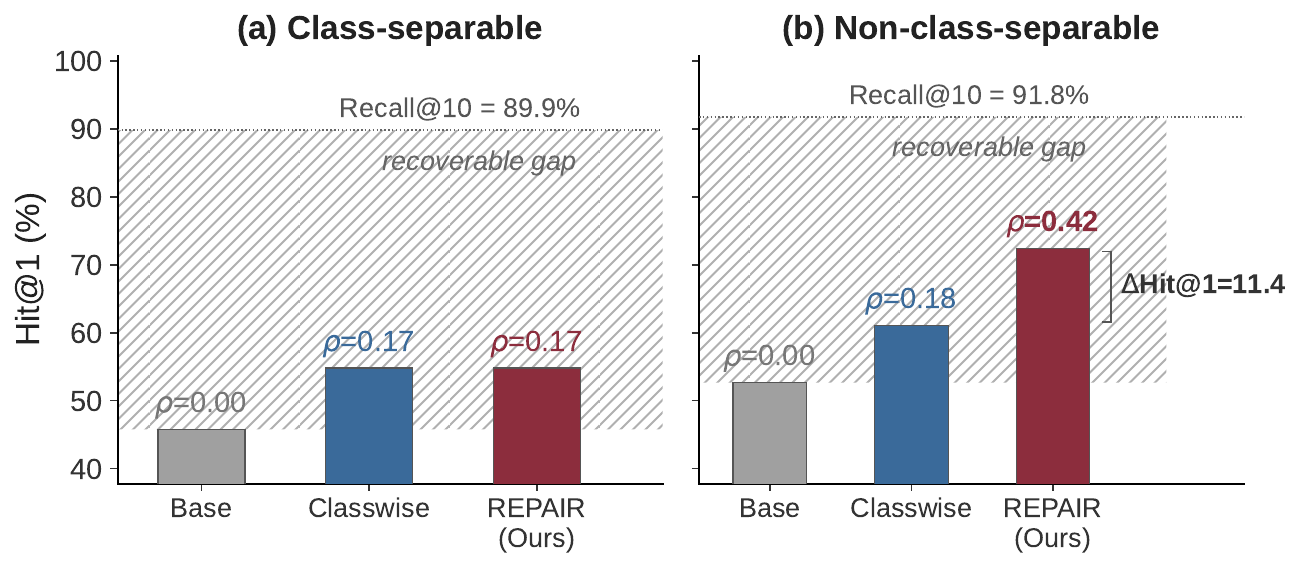}
\vspace{-6pt}
\caption{Synthetic validation ($K{=}100$, $k{=}10$).
\textbf{(a)}~Class-separable regime: classwise correction suffices ($\rho_k = 0.17$ for both), confirming Proposition~\ref{prop:classwise-special}.
\textbf{(b)}~Non-class-separable regime with contradictory pairs: \repair{} closes $2.3\times$ the recoverable gap ($\rho_k = 0.42$ vs.\ $0.18$), consistent with Theorem~\ref{thm:pairwise-necessary}.}
\label{fig:regime}
\end{figure}

%% file: algorithm.tex
\section{Algorithm}
\label{sec:algorithm}

Theorem~\ref{thm:pairwise-necessary} identifies the missing degree of freedom: the correction for a label cannot depend only on the label identity; it must also vary with the competing labels on the shortlist. More precisely, the excess risk of any plug-in reranker is controlled by how well it estimates the shortlist posterior (Appendix~\ref{proof:plugin}), which in turn depends on approximating both the classwise and pairwise components of the residual. 
This section translates that structural insight into a concrete, lightweight algorithm.
\subsection{Score model}
\label{sec:score-model}

We model the reranker score as
\begin{equation}
\label{eq:score-model}
r_y(x,S) \;=\; g_y(x) \;+\; \underbrace{a_y}_{\text{classwise}}
\;+\; \underbrace{\ell_y(x,S)}_{\text{pairwise}}, \qquad y \in S,
\end{equation}
where $\hat{\beta}_y := a_y + \ell_y(x,S)$ is our estimate of the residual correction $\beta^\star_y$ (Eq.~\ref{eq:residual}).
The classwise term $a_y$ targets the class-separable component of $\beta^\star_y$ (Definition~\ref{def:class-separable});
the special choice $a_y = -\tau\log\pi_y$ recovers logit adjustment, but we learn $a_y$ from covered calibration examples to obtain the strongest fixed-classwise correction.
Even without pairwise term, this 
learned classwise correction outperforms the closed-form logit adjustment offset on most datasets (Table~\ref{tab:main_combined}), because it captures per-class biases beyond frequency.
The pairwise term $\ell_y$ targets the remainder that fixed classwise corrections structurally miss (Theorem~\ref{thm:pairwise-necessary}).

The pairwise correction averages over all rivals in the shortlist:
\begin{equation}
\label{eq:pairwise-correction}
\ell_y(x,S) \;=\; \frac{1}{|S|-1} \sum_{j \in S \setminus \{y\}} \theta^\top \phi(x,y,j).
\end{equation}
Each shortlisted rival $j$ contributes a correction to label $y$ through the shared parameter vector~$\theta$ and the pairwise feature vector $\phi(x,y,j)$.

\paragraph{Two operational modes.}
This decomposition gives REPAIR two modes rather than a single fixed recipe. When
the calibration data behave class-separably (Definition~\ref{def:class-separable}),
the learned classwise term $a_y$ alone recovers the Bayes ordering and the pairwise
term can be omitted; this is the \emph{Classwise mode}. When the residual is
non-class-separable, the \emph{Pairwise mode} adds $\ell_y(x,S)$ to capture the
competition structure a fixed offset cannot. Classwise is therefore not merely a
baseline but the class-separable special case of the same framework, and the
pairwise term is optional: it is warranted precisely when a diagnostic check on
calibration data shows a non-class-separable behavior.

\subsection{Pairwise features}
\label{sec:features}

The feature vector $\phi(x,y,j)$ encodes how label $y$ competes against rival $j$ in context~$x$.
Four \emph{competition features} are computed from the base model's output: the \emph{score gap} $g_y(x) - g_j(x)$ (the negative of the threshold $t(x;\,y,j)$ in Eq.~\ref{eq:threshold}), the \emph{rank gap} $\mathrm{rank}(j) - \mathrm{rank}(y)$, the \emph{log-frequency ratio} $\log n_y - \log n_j$, where $n_y$ denotes the number of training examples for class $y$, and the log-probability ratio $\log P(y|S)/P(j|S)$.
\emph{Label-similarity features} encode how confusable two classes are: taxonomic distance for iNaturalist, WordNet path similarity for ImageNet-LT, and HPO phenotype similarity for the rare-disease benchmarks (Appendix~\ref{app:features}).

\subsection{Training and shrinkage}
\label{sec:training}

All parameters are estimated on \emph{covered} examples from a held-out calibration set with ($y_i \in S_i$).
The base model is never modified.
We fit $a$ and $\theta$ jointly by maximizing a conditional multinomial logit on the shortlist:
\begin{equation}
\label{eq:softmax}
q_{\theta,a}(y \mid x,S) \;=\; \frac{\exp\bigl(r_y(x,S)\bigr)}{\sum_{j \in S} \exp\bigl(r_j(x,S)\bigr)}, \quad y \in S.
\end{equation}
The training objective is the regularized conditional log-likelihood:
\begin{equation}
\label{eq:training-obj}
\mathcal{L}(a,\theta;\,\mathcal{D})
  \;=\; \sum_{(x_i,S_i,y_i) \in \mathcal{D}}
        \log q_{\theta,a}(y_i \mid x_i,S_i)
  \;-\; \lambda_\theta \|\theta\|_2^2
  \;-\; \lambda_a \|a\|_2^2,
\end{equation}
where $\mathcal{D} = \{(x_i,S_i,y_i) : y_i \in S_i\}$ is the set of covered examples.

\paragraph{Shrinkage.}
For rare classes with very few covered examples, $a_y$ can be noisy.
We stabilize via empirical-Bayes shrinkage toward a frequency-group mean:
\begin{equation}
\label{eq:shrinkage}
a^\star_y \leftarrow (1 - \lambda_y)\,\hat{a}_y + \lambda_y\,\mu_{b(y)},
\qquad \lambda_y = \frac{\sigma^2_y}{\sigma^2_y + \nu^2_{b(y)}},
\end{equation}
where $b(y)$ indexes a frequency group, $\mu_{b(y)}$ is the group mean, and $\sigma^2_y$, $\nu^2_{b(y)}$ are within-class and between-group variance estimates (Appendix~\ref{app:shrinkage}).
The complete procedure is summarized in Algorithm~\ref{alg:repair} (Appendix~\ref{app:algorithm}).

%% file: experiment.tex
\section{Experiments}
\label{sec:experiments}

\subsection{Setup}
\label{sec:setup}

\paragraph{Datasets and base models.}
We evaluate on nine benchmarks spanning text classification, image classification,
species recognition, and rare-disease diagnosis.
The four pure-text NLP benchmarks are
\textsc{GoEmotions} (28 classes), \textsc{MASSIVE} (60 classes),
\textsc{HuffPost-LT} (42 classes), and \textsc{Few-NERD} (66 classes),
each using a Qwen3.5-2B base classifier fine-tuned with LoRA.
The remaining benchmarks are image classification
(\textsc{ImageNet-LT}, 1{,}000 classes; \textsc{Places-LT}, 365 classes),
species recognition (\textsc{iNaturalist}~2018, 8{,}142 classes),
and rare-disease diagnosis
(\textsc{GMDB}, 508 classes; \textsc{RareBench}, 508 classes, out-of-distribution).
Base models range from ResNet-50/152 with cRT classifiers to a fine-tuned Qwen3.5-0.8B
for the clinical benchmarks and Qwen3.5-2B for the NLP benchmarks;
full details are in Appendix~\ref{app:datasets}.
In all cases, REPAIR operates on a fixed top-$k$ shortlist ($k{=}10$)
produced by the base model; the base model is never modified.
We designate the bottom 80\% of classes by training-set frequency
as \emph{rare} and the top 20\% as \emph{frequent}.

\paragraph{Baselines.}
All methods share the same shortlist.
\textbf{Base} ranks by the raw scores $g_y(x)$.
\textbf{LogitAdj}~\citep{menon2021longtail} adds $-\tau\log\pi_y$
with $\tau$ tuned on the calibration set.
\textbf{$\tau$-norm}~\citep{kang2020decoupling}
divides by the classifier weight norm $\|w_y\|^\tau$.
\textbf{Classwise} learns per-class offsets $a_y$
via Eq.~\ref{eq:training-obj} with shrinkage (Eq.~\ref{eq:shrinkage}),
setting $\theta=0$.
\textbf{REPAIR} adds the pairwise term $\ell_y(x,S)$.


\paragraph{Evaluation schemes.}
Reranker parameters are estimated on a held-out calibration set
and evaluated on an independent test set; the two never overlap. Each experiment is repeated over five trials, each subsampling $80\% $ of the calibration set with a distinct seed.
We report mean $\pm$ std for methods with learned parameters (Classwise, REPAIR) and single values for LogitAdj and $\tau$-norm, whose corrections are deterministic given the training-set frequency distribution and a fixed $\tau$ (tuned once on the full calibration set), and therefore do not vary across calibration subsamples. Significance is assessed via a sign test ($\geq 4/5$ wins);
full results are in Appendix~\ref{app:signtest}.
All metrics are conditioned on coverage ($Y \in S$).
We report Hit@1, Hit@3, MRR, Rare Hit@1, Frequent Hit@1, and the hardest-rival flip rate (HFR; Definition~\ref{def:hfr}), which measures whether the reranker can beat the strongest wrong competitor on inputs the base model misranks.

\begin{definition}[Hardest-rival flip rate] 
\label{def:hfr}
The probability that the reranker ranks the true label above the base model's top-scoring wrong class, among covered examples the base model misranks:
\[
\mathrm{HFR}(\hat{y}) = P\bigl(r_Y > r_{h(X,Y)} \mid C,\; 
\hat{y}_{\mathrm{base}} \neq Y\bigr),
\]
where $h(x,y) := \arg\max_{j \in S(x)\setminus\{y\}} g_j(x)$.
\end{definition}

\subsection{Results}
\label{sec:results}
Table~\ref{tab:main_combined} reports the comparison on the vision and clinical benchmarks, and Table~\ref{tab:main_nlp} the pure-text NLP benchmarks. The decomposition makes a direct prediction: on datasets where the base model's score gaps  are stable across inputs, classwise correction should suffice 
and the pairwise term should add little; 
on datasets where score gaps vary widely, 
the pairwise term should provide substantial gains.

\paragraph{Near-class-separable regime (vision benchmarks).}
On iNaturalist, ImageNet-LT, and Places-LT, REPAIR provides small but consistent Hit@1 improvements over Classwise ($+0.2\%$, $+0.2\%$, $+0.4\%$), with the learned Classwise correction competitive with or outperforming LogitAdj on all three.
HFR nearly doubles on iNaturalist ($0.074\to0.138$) and ImageNet-LT ($0.336\to0.378$). On Places-LT, LogitAdj and $\tau$-norm both degrade Hit@1 relative to Base, while Classwise and REPAIR both exceed it, confirming that learned offsets are more robust than closed-form corrections. The decomposition identifies these datasets 
as predominantly classwise: 
most ranking errors are correctable 
by a fixed per-class offset.

\paragraph{Non-class-separable regime (rare-disease benchmarks).}
On GMDB, REPAIR improves Hit@1 from $71.4$ to $72.3$ and HFR from $0.441$ to
$0.510$. The aggregate Hit@1 gain is driven by the model-derived competition
features rather than by HPO specifically; removing the HPO feature leaves every
REPAIR metric unchanged by less than $0.1$ (Appendix~\ref{app:hpo-ablation}).
HPO instead shapes the composition of the top-$k$ differential diagnosis,
promoting the correct diagnosis while keeping phenotypically related conditions
near the top (Appendix~\ref{app:casestudy}).
On RareBench (OOD), Hit@1 rises from $66.2$ to $85.6$ ($+19.4$\,\%)
and HFR from $0.215$ to $0.708$.
Two factors drive this:
the base model operates out-of-distribution
(trained on GMDB with image and text, tested on text),
which amplifies threshold variation across inputs;
and HPO phenotype similarity gives the pairwise term
discriminative signal the base model has no access to.
The RareBench test set contains only 32 covered examples,
so we verify this result with 10K-resample bootstrap:
the 95\% CI for the gain over Classwise is 
$[+6.2,\; +31.2](\%)$($P{=}0.999$; Appendix~\ref{app:rarebench-ci}).

\paragraph{Pure-text NLP classifiers.}
The decomposition is modality-agnostic: it applies to any classifier that produces class scores, including language models with a classification head. We verify this on four pure-text NLP
benchmarks (GoEmotions, MASSIVE, HuffPost-LT, Few-NERD), each using a Qwen3.5-2B
base classifier fine-tuned with LoRA under the identical protocol. The same
contrast holds: on the class-separable \textsc{GoEmotions}, Classwise already
suffices and REPAIR adds little ($+0.05$ Hit@1 over Classwise), whereas on the
non-class-separable \textsc{MASSIVE}, \textsc{HuffPost-LT}, and \textsc{Few-NERD}
the pairwise term yields clear gains ($+1.78$, $+0.84$, $+1.05$ Hit@1, and
$+5.53$, $+2.04$, $+1.83$ Rare Hit@1 over Classwise). Full results are in
Appendix~\ref{app:nlp}.

\paragraph{Head--tail trade-off.}
On vision benchmarks, improving Rare Hit@1 comes at some cost 
to Frequent Hit@1. Most of this shift originates from the 
classwise correction itself; the pairwise term adds only a 
small further change. On RareBench, both strata improve 
simultaneously. 

\paragraph{Component ablation.}
In the synthetic setting 
(Figure~\ref{fig:ablation-synthetic}, Appendix~\ref{app:ablation}),
CW-only matches REPAIR in the class-separable regime 
while PW-only adds nothing;
in the non-class-separable regime, 
the full model outperforms both 
($72.5$ vs.\ $61.1$ CW-only, $72.1$ PW-only).
On real data (Appendix~\ref{app:ablation}),
the relative importance of the two components 
shifts across the spectrum.
On iNaturalist, 
PW-only alone matches REPAIR 
(Rare Hit@1: $42.1$),
suggesting that the pairwise features 
implicitly absorb the classwise correction.
On ImageNet-LT, Places-LT and GMDB, both components contribute incrementally.
On RareBench, neither component alone 
approaches the full model 
(CW-only $58.8$, PW-only $47.1$, REPAIR $88.2$; single random seed):
the classwise term calibrates the score scale 
while the pairwise term redirects rankings,
and joint estimation is essential 
for the two corrections to reinforce each other.

\subsection{Quintile analysis}
\label{sec:quintile}

\begin{figure}[!ht]
\centering
\includegraphics[width=0.85\textwidth]{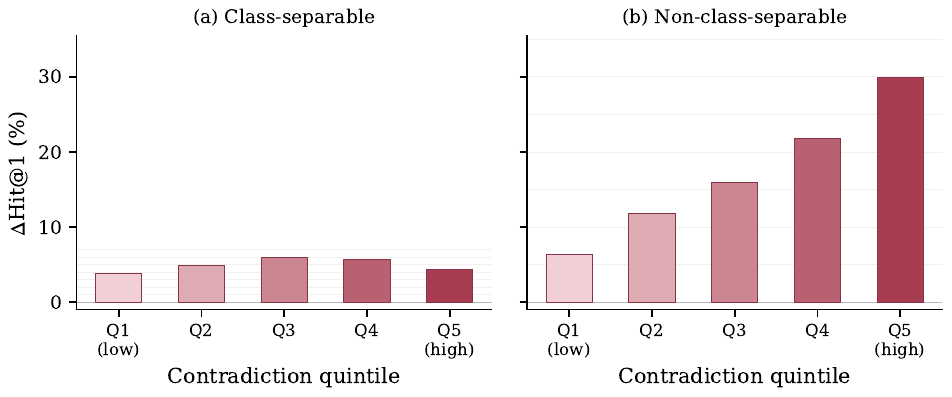}
\vspace{-6pt}
\caption{$\Delta$Hit@1 of REPAIR over Classwise 
by mean $D_y$ quintile 
(synthetic, $K{=}100$, $k{=}10$, rare classes).
\textbf{(a)}~Class-separable: almost flat across quintiles.
\textbf{(b)}~Non-class-separable: monotone increasing from Q1 to Q5.}
\label{fig:quintile-synthetic}
\end{figure}

We stratify rare classes by their mean threshold 
dispersion $D_y$ (Eq.~\ref{eq:threshold-dispersion}) 
into five equal-sized quintiles 
and measure the $\Delta$Hit@1 of REPAIR over Classwise 
within each bin.
In the class-separable setting~(a), 
the gain is uniformly small and flat, which means $D_y$ 
does not modulate the pairwise contribution 
when the residual is approximately classwise.
In the non-class-separable setting~(b), 
the gain increases monotonically 
from Q1 ($+6\%$) to Q5 ($+30\%$).
On real data (Appendix~\ref{app:quintile}), the same pattern holds: on GMDB, Q5 shows the largest gain while Q1 shows the smallest, consistent with the non-class-separable regime; on iNaturalist, the gain is uniformly small and the quintile gradient is flat, consistent with a near-class-separable residual. ImageNet-LT and Places-LT fall between these extremes.

\begin{table*}[t]
\centering
\caption{Reranking results on three visual recognition and two rare-disease diagnosis benchmarks (shortlist size $k=10$). All metrics are conditioned on the true label appearing in the shortlist; arrows indicate higher is better. Rare and Frequent denote the bottom $80\%$ and top $20\%$ of classes by training-set frequency. Classwise is our framework without the pairwise term; both Classwise and \repair{} report mean{\scriptsize$\pm$std} over 5 calibration subsamples. Best per dataset in \textbf{bold}. Row shading: \colorbox{blue!6}{blue} = near-class-separable (vision benchmarks), \colorbox{red!8}{red} = non-class-separable (rare-disease benchmarks).}
\label{tab:main_combined}
\vspace{2pt}
\setlength{\tabcolsep}{3.5pt}
\renewcommand{\arraystretch}{1.10}
\resizebox{\textwidth}{!}{\begin{tabular}{l l l ccc cc c}
\toprule
& & & \multicolumn{3}{c}{\textit{Overall}} & \multicolumn{2}{c}{\textit{Stratified Hit@1}} & \\
\cmidrule(lr){4-6} \cmidrule(lr){7-8}
Task & Dataset & Method & Hit@1$\uparrow$ & Hit@3$\uparrow$ & MRR$\uparrow$ & Rare$\uparrow$ & Freq.$\uparrow$ & HFR$\uparrow$ \\
\midrule
\multirow{5}{*}{\rotatebox[origin=c]{90}{\shortstack{Species\\Recog.}}}
 & \multirow{5}{*}{\shortstack[l]{\textsc{iNaturalist}\\[-1pt]{\scriptsize(ResNet-50)}}}
 & Base             & $47.0$ & $71.7$ & $.627$ & $37.5$ & $\mathbf{60.0}$ & --    \\
 & & LogitAdj         & $\mathbf{49.0}$ & $74.7$ & $.648$ & $41.7$ & $58.9$ & $.097$ \\
 & & $\tau$-norm      & $48.1$ & $73.7$ & $.640$ & $40.0$ & $59.2$ & $.105$ \\
 & & Classwise        & $48.8${\tiny$\pm.02$} & $74.5${\tiny$\pm.01$} & $.646${\tiny$\pm.000$} & $41.1${\tiny$\pm.02$} & $59.3${\tiny$\pm.01$} & $.074${\tiny$\pm.001$} \\
 & & \cellcolor{blue!10} REPAIR (Ours)
   & \cellcolor{blue!10} $\mathbf{49.0}${\tiny$\pm.01$}
   & \cellcolor{blue!10} $\mathbf{75.6}${\tiny$\pm.02$}
   & \cellcolor{blue!10} $\mathbf{.652}${\tiny$\pm.000$}
   & \cellcolor{blue!10} $\mathbf{42.2}${\tiny$\pm.02$}
   & \cellcolor{blue!10} $58.4${\tiny$\pm.02$}
   & \cellcolor{blue!10} $\mathbf{.138}${\tiny$\pm.002$} \\
\midrule
\multirow{5}{*}{\rotatebox[origin=c]{90}{\shortstack{Image\\Classif.}}}
 & \multirow{5}{*}{\shortstack[l]{\textsc{ImageNet-LT}\\[-1pt]{\scriptsize(ResNet-50)}}}
 & Base             & $54.4$ & $78.6$ & $.689$ & $47.4$ & $\mathbf{76.7}$ & --    \\
 & & LogitAdj         & $61.5$ & $85.3$ & $.748$ & $59.5$ & $68.1$ & $.321$ \\
 & & $\tau$-norm      & $60.5$ & $84.8$ & $.741$ & $58.1$ & $68.2$ & $.299$ \\
 & & Classwise        & $61.8${\tiny$\pm.02$} & $85.4${\tiny$\pm.02$} & $.750${\tiny$\pm.000$} & $59.9${\tiny$\pm.05$} & $67.8${\tiny$\pm.1$} & $.336${\tiny$\pm.001$} \\
 & & \cellcolor{blue!10} REPAIR (Ours)
   & \cellcolor{blue!10} $\mathbf{62.0}${\tiny$\pm.04$}
   & \cellcolor{blue!10} $\mathbf{85.4}${\tiny$\pm.03$}
   & \cellcolor{blue!10} $\mathbf{.751}${\tiny$\pm.000$}
   & \cellcolor{blue!10} $\mathbf{60.8}${\tiny$\pm.05$}
   & \cellcolor{blue!10} $65.7${\tiny$\pm.1$}
   & \cellcolor{blue!10} $\mathbf{.378}${\tiny$\pm.001$} \\
\midrule
\multirow{5}{*}{\rotatebox[origin=c]{90}{\shortstack{Scene\\Recog.}}}
 & \multirow{5}{*}{\shortstack[l]{\textsc{Places-LT}\\[-1pt]{\scriptsize(ResNet-152)}}}
 & Base             & $44.7$ & $72.1$ & $.617$ & $46.3$ & $\mathbf{37.9}$ & --    \\
 & & LogitAdj         & $44.5$ & $71.5$ & $.613$ & $47.9$ & $30.4$ & $.072$ \\
 & & $\tau$-norm      & $19.7$ & $45.5$ & $.397$ & $24.5$ & $0.0$ & $\mathbf{.405}$ \\
 & & Classwise        & $45.1${\tiny$\pm.1$} & $72.4${\tiny$\pm.2$} & $.620${\tiny$\pm.000$} & $46.9${\tiny$\pm.1$} & $37.4${\tiny$\pm.3$} & $.028${\tiny$\pm.001$} \\
 & & \cellcolor{blue!10} REPAIR (Ours)
   & \cellcolor{blue!10} $\mathbf{45.5}${\tiny$\pm.1$}
   & \cellcolor{blue!10} $\mathbf{72.5}${\tiny$\pm.1$}
   & \cellcolor{blue!10} $\mathbf{.622}${\tiny$\pm.001$}
   & \cellcolor{blue!10} $\mathbf{48.2}${\tiny$\pm.1$}
   & \cellcolor{blue!10} $34.1${\tiny$\pm.3$}
   & \cellcolor{blue!10} $.072${\tiny$\pm.001$} \\
\midrule
\multirow{9}{*}{\rotatebox[origin=c]{90}{\shortstack{Rare Disease\\Diagnosis}}}
 & \multirow{5}{*}{\shortstack[l]{\textsc{GMDB}\\[-1pt]{\scriptsize(Qwen-0.8B)}}}
 & Base             & $64.0$ & $79.3$ & $.741$ & $62.9$ & $64.7$ & --    \\
 & & LogitAdj         & $70.6$ & $85.5$ & $.797$ & $78.8$ & $65.1$ & $.422$ \\
 & & $\tau$-norm      & $67.4$ & $81.6$ & $.769$ & $66.4$ & $\mathbf{68.0}$ & $.168$ \\
 & & Classwise        & $71.4${\tiny$\pm.1$} & $85.8${\tiny$\pm.0$} & $.801${\tiny$\pm.001$} & $78.5${\tiny$\pm.1$} & $66.5${\tiny$\pm.2$} & $.441${\tiny$\pm.004$} \\
 & & \cellcolor{red!10} REPAIR (Ours)
   & \cellcolor{red!10} $\mathbf{72.3}${\tiny$\pm.3$}
   & \cellcolor{red!10} $\mathbf{85.8}${\tiny$\pm.2$}
   & \cellcolor{red!10} $\mathbf{.806}${\tiny$\pm.001$}
   & \cellcolor{red!10} $\mathbf{79.6}${\tiny$\pm.3$}
   & \cellcolor{red!10} $67.4${\tiny$\pm.5$}
   & \cellcolor{red!10} $\mathbf{.510}${\tiny$\pm.005$} \\
\cmidrule(l){2-9}
 & \multirow{5}{*}{\shortstack[l]{\textsc{RareBench}\\[-1pt]{\scriptsize(Qwen-0.8B, text)}}}
 & Base             & $59.4$ & $75.0$ & $.712$ & $41.2$ & $80.0$ & --   \\
 & & LogitAdj         & $62.5$ & $71.9$ & $.712$ & $52.9$ & $73.3$ & $.385$ \\
 & & $\tau$-norm      & $56.2$ & $78.1$ & $.702$ & $41.2$ & $73.3$ & $.000$ \\
 & & Classwise        & $66.2${\tiny$\pm 1.2$} & $91.2${\tiny$\pm 2.3$} & $.789${\tiny$\pm.006$} & $54.1${\tiny$\pm 2.4$} & $80.0${\tiny$\pm.0$} & $.215${\tiny$\pm.031$} \\
 & & \cellcolor{red!10} REPAIR (Ours)
   & \cellcolor{red!10} $\mathbf{85.6}${\tiny$\pm 4.2$}
   & \cellcolor{red!10} $\mathbf{92.5}${\tiny$\pm 1.5$}
   & \cellcolor{red!10} $\mathbf{.899}${\tiny$\pm.026$}
   & \cellcolor{red!10} $\mathbf{81.2}${\tiny$\pm 4.4$}
   & \cellcolor{red!10} $\mathbf{90.7}${\tiny$\pm 5.3$}
   & \cellcolor{red!10} $\mathbf{.708}${\tiny$\pm.090$} \\
\bottomrule
\end{tabular}}
\end{table*}

%% file: conclusion.tex
\section{Conclusion}
\label{sec:conclusion}

In this work, we proposed a residual decomposition framework for post-hoc reranking under long-tailed class distributions. Unlike logit adjustment, which applies a fixed per-class offset derived from class frequencies, our framework decomposes the Bayes-optimal correction into a learned classwise term and an input-dependent pairwise term, with a formal characterization of when the classwise component alone is sufficient and when it is not. Empirically, the learned classwise correction already outperforms logit adjustment on most datasets, and the pairwise term provides further gains proportional to the degree of threshold variation across nine benchmarks spanning text classification, visual recognition, and rare-disease diagnosis, with especially large improvements in the rare-disease setting where phenotypically similar conditions create strong pairwise competition. The current pairwise features are compact and domain-interpretable, but richer representations learned from the base model's embeddings could capture competition patterns that hand-crafted features miss. Overall, these results demonstrate that in long-tailed settings, decomposing the residual correction into classwise and pairwise components offers a principled and lightweight alternative to purely frequency-based post-hoc adjustments, particularly in domains where local competition between similar classes causes ranking errors that no fixed offset can resolve.

%% file: appendix.tex
\appendix
\section{Additional Related Work}
\label{app:add_related}

\paragraph{Training-time corrections.}
Training time methods modify the loss, sampling or architecture during learning. Re-weighing approaches adjust per-class loss contributions based on class frequency or sample difficulties ~\citep{focalloss,cbloss}, while margin-based methods enforce larger-decision margin for rare classes~\citep{marginloss}. At sampling level, decoupled training learns representation with standard sampling and re-balances the classifier~\citep{kang2020decoupling,improvingcalibration}. Ensemble methods in addition offer structural solution by training multiple experts with complementary skills across frequent classes~\citep{zhang2022selfsupervised,wang2021longtailed}. 
\section{Notation Reference}
\label{app:notation}
 
Table~\ref{tab:notation} summarizes the principal symbols used throughout the paper.
 
\begin{table}[h]
\centering
\caption{Summary of notation.}
\label{tab:notation}
\setlength{\tabcolsep}{6pt}
\renewcommand{\arraystretch}{1.15}
\begin{tabular}{l l l}
\toprule
Symbol & Definition & Location \\
\midrule
\multicolumn{3}{l}{\textit{Setup}} \\
$X$ & Input & \S\ref{sec:prelim} \\
$Y \in [K]$ & True label & \S\ref{sec:prelim} \\
$\pi_y$ & Class prior $P(Y{=}y)$ & \S\ref{sec:prelim} \\
$g_y(x)$ & Base-model score for class $y$ & \S\ref{sec:prelim} \\
$k$ & Shortlist size & \S\ref{sec:prelim} \\
$S(x)$ & Top-$k$ shortlist under $g$ & \S\ref{sec:prelim} \\
$C$ & Coverage event $\{Y \in S(X)\}$ & \S\ref{sec:prelim} \\
\midrule
\multicolumn{3}{l}{\textit{Theory (\S\ref{sec:framework})}} \\
$p_S(y \mid x)$ & Shortlist posterior & Eq.~(\ref{eq:shortlist-posterior}) \\
$s^\star_y(x,S)$ & Bayes-optimal shortlist score & Eq.~(\ref{eq:bayes-score}) \\
$\beta^\star_y(x,S)$ & Residual correction $s^\star_y - g_y$ & Eq.~(\ref{eq:residual}) \\
$\Delta_{uv}(x,S)$ & Pairwise gap $\beta^\star_u - \beta^\star_v$ & Eq.~(\ref{eq:pairwise-gap}) \\
$t(x;\,u,v)$ & Base-score threshold $g_v(x) - g_u(x)$ & Eq.~(\ref{eq:threshold}) \\
$D_y(S)$ & Threshold dispersion & Eq.~(\ref{eq:threshold-dispersion}) \\
$\rho_k$ & Recoverable gap closure & Eq.~(\ref{eq:gap-closed}) \\
\midrule
\multicolumn{3}{l}{\textit{Algorithm (\S\ref{sec:algorithm})}} \\
$r_y(x,S)$ & Reranker score $g_y + \hat{\beta}_y$ & Eq.~(\ref{eq:score-model}) \\
$\hat{\beta}_y$ & Estimated residual $a_y + \ell_y$ & \S\ref{sec:score-model} \\
$a_y$ & Learned classwise offset & Eq.~(\ref{eq:score-model}) \\
$\ell_y(x,S)$ & Pairwise correction & Eq.~(\ref{eq:score-model}) \\
$\theta$ & Shared pairwise parameter & Eq.~(\ref{eq:pairwise-correction}) \\
$\phi(x,y,j)$ & Pairwise feature vector & \S\ref{sec:features} \\
$\nu^2_{b(y)},\;\sigma^2_y$ & Shrinkage variances & Eq.~(\ref{eq:shrinkage}) \\
\midrule
\multicolumn{3}{l}{\textit{Evaluation (\S\ref{sec:experiments})}} \\
Hit@$k$ & $P(\mathrm{rank}_r(Y) \leq k \mid C)$ & \S\ref{sec:setup} \\
MRR & $\mathbb{E}[\mathrm{rank}_r(Y)^{-1} \mid C]$ & \S\ref{sec:setup} \\
HFR & Hardest-rival flip rate & Def.~\ref{def:hfr} \\
\bottomrule
\end{tabular}
\end{table}
 
\section{Theoretical Supplement}
\label{app:theory}

\subsection{Restricted posterior on the shortlist}
\label{app:restricted-posterior}

Because $S(x)$ is a deterministic function of $x$, the coverage probability given $X=x$ is
\begin{equation}
\alpha(x) := P(C \mid X=x) = \sum_{j \in S(x)} P(Y=j \mid X=x).
\end{equation}
Whenever $\alpha(x) > 0$,
\begin{equation}
p_S(y \mid x) = P(Y=y \mid X=x, C) = \frac{P(Y=y \mid X=x) \cdot \mathbf{1}\{y \in S(x)\}}{\alpha(x)}.
\end{equation}
Hence the Bayes-optimal shortlist score is
\begin{equation}
s^\star_y(x,S) = \log p_S(y \mid x) = \log P(Y=y \mid X=x) - \log\alpha(x), \qquad y \in S(x),
\end{equation}
so only the relative values across $y \in S(x)$ matter for ranking.
The term $\log\alpha(x)$ is shared by all labels in the shortlist and therefore cancels out of any pairwise comparison.

\subsection{Weighted variant}
\label{app:weighted}

The paper uses ordinary top-1 risk for clarity.
If one instead optimizes a class-weighted risk with weights $w_y > 0$, then the Bayes-optimal shortlist score becomes $\log w_y + \log p_S(y \mid x)$.
All main arguments carry over verbatim after adding $\log w_y$ to the ideal score.
We keep the unweighted version in the main text because it is easier to read.

\subsection{Unconditional form of the recoverable gap closure}
\label{sec:rho-unconditional}

When expressed in unconditional accuracy, the recoverable gap closure~(\ref{eq:gap-closed}) takes the equivalent form
\begin{equation}
\rho_k(\hat{y}) = \frac{\mathrm{Acc}(\hat{y}) - \mathrm{Acc}(\hat{y}_{\mathrm{base}})}{P(C) - \mathrm{Acc}(\hat{y}_{\mathrm{base}})},
\end{equation}
since $\mathrm{Hit@1}(\hat{y}) = \mathrm{Acc}(\hat{y})/P(C)$ and the $P(C)$ factors cancel.
This equivalence holds because all rerankers share the same shortlist and hence the same coverage $P(C)$.

\subsection{Proofs}
\label{app:proofs}

\subsubsection{Bayes-optimal reranking on the shortlist}
\label{proof:bayes-shortlist}

\begin{proposition}[Bayes rule on a fixed shortlist]
Among all measurable rules $\hat{y}(X) \in S(X)$, a Bayes-optimal reranker can be chosen pointwise as $\hat{y}^\star(x) \in \arg\max_{y \in S(x)} p_S(y \mid x)$.
Equivalently, any score function strictly increasing in $p_S(y \mid x)$ yields the same Bayes ordering.
\end{proposition}

\begin{proof}
Fix $x$.
Among all measurable rules $r(x) \in S(x)$, the conditional probability of being correct is
\[
P(r(X)=Y \mid X=x, C) = \sum_{y \in S(x)} \mathbf{1}\{r(x)=y\}\, p_S(y \mid x) = p_S(r(x) \mid x).
\]
Since $r(x)$ is a deterministic choice, the indicator is nonzero for exactly one value of $y$, and the sum collapses to $p_S(r(x) \mid x)$.
Therefore the conditional risk given $X=x$ and $C$ is minimized by any $r(x) \in \arg\max_{y \in S(x)} p_S(y \mid x)$.
Since this argument is pointwise in $x$, any measurable selector of that argmax is Bayes-optimal on covered examples.
\end{proof}

\subsubsection{Plug-in estimation error}
\label{proof:plugin}

\begin{proposition}[Plug-in control]
Let $\hat{p}_S(\cdot \mid x)$ be an estimate of $p_S(\cdot \mid x)$, and let $\hat{y}_{\mathrm{plug}}(x) \in \arg\max_{y \in S(x)} \hat{p}_S(y \mid x)$.
Then the conditional excess risk on covered examples satisfies
\begin{equation}
P(\hat{y}_{\mathrm{plug}}(X) \neq Y \mid C) - P(\hat{y}^\star(X) \neq Y \mid C) \leq 2\,\mathbb{E}[\|\hat{p}_S(\cdot \mid x) - p_S(\cdot \mid x)\|_1 \mid C].
\end{equation}
\end{proposition}

\begin{lemma}[Covered excess-risk identity]
For any measurable shortlist rule $\hat{y}(X) \in S(X)$,
\[
P(\hat{y}(X) \neq Y \mid C) - P(\hat{y}^\star(X) \neq Y \mid C) = \mathbb{E}[p_S(\hat{y}^\star(X) \mid x) - p_S(\hat{y}(X) \mid x) \mid C].
\]
\end{lemma}

\begin{proof}
Condition on $X=x$ and $C$.
The conditional error rates are $1-p_S(\hat{y}(x) \mid x)$ and $1-p_S(\hat{y}^\star(x) \mid x)$, respectively.
Subtracting:
\[
\bigl(1-p_S(\hat{y}(x) \mid x)\bigr) - \bigl(1-p_S(\hat{y}^\star(x) \mid x)\bigr) = p_S(\hat{y}^\star(x) \mid x) - p_S(\hat{y}(x) \mid x).
\]
Taking expectations over $X$ given $C$ yields the identity.
\end{proof}

\begin{proof}[Proof of Proposition (Plug-in control)]
Using the lemma,
\[
\Delta := P(\hat{y}_{\mathrm{plug}}(X) \neq Y \mid C) - P(\hat{y}^\star(X) \neq Y \mid C) = \mathbb{E}\bigl[p_S(\hat{y}^\star(X) \mid x) - p_S(\hat{y}_{\mathrm{plug}}(X) \mid x) \mid C\bigr].
\]
For each $x$, we add and subtract $\hat{p}_S$:
\begin{align*}
p_S(\hat{y}^\star(x) \mid x) - p_S(\hat{y}_{\mathrm{plug}}(x) \mid x)
&= \underbrace{p_S(\hat{y}^\star(x) \mid x) - \hat{p}_S(\hat{y}^\star(x) \mid x)}_{\text{(i)}} \\
&\quad + \underbrace{\hat{p}_S(\hat{y}^\star(x) \mid x) - \hat{p}_S(\hat{y}_{\mathrm{plug}}(x) \mid x)}_{\text{(ii)}} \\
&\quad + \underbrace{\hat{p}_S(\hat{y}_{\mathrm{plug}}(x) \mid x) - p_S(\hat{y}_{\mathrm{plug}}(x) \mid x)}_{\text{(iii)}}.
\end{align*}
Term (ii) $\leq 0$ by definition of $\hat{y}_{\mathrm{plug}}$.
Each of terms (i) and (iii) is bounded by $\|\hat{p}_S(\cdot \mid x) - p_S(\cdot \mid x)\|_1$.
Taking conditional expectations given $C$ completes the proof.
\end{proof}

\subsubsection{Proof of Proposition~\ref{prop:classwise-special}}
\label{proof:classwise-special}

\begin{proof}
If $\beta^\star_y(x,S) = a_y + c(x,S)$ for every covered triple $(x,S,y)$ and some shortlist-shared scalar $c(x,S)$, then
\[
s^\star_y(x,S) = g_y(x) + a_y + c(x,S) \quad \text{for all } y \in S.
\]
Therefore the fixed-offset score $g_y(x) + a_y$ differs from the Bayes-optimal shortlist score only by the shortlist-shared constant $c(x,S)$.
Since adding the same constant to all shortlisted labels does not change the ranking, the ordering induced by $g_y(x)+a_y$ is Bayes-optimal on every covered example.
\end{proof}

\subsubsection{Proof of Theorem~\ref{thm:pairwise-necessary}}
\label{proof:pairwise-necessary}

\begin{proof}
Assume for contradiction that some fixed-classwise model $g_\cdot(x) + a_\cdot$ agrees with the Bayes ordering in both contexts.

Because the Bayes rule ranks $u$ above $v$ at $x$, the fixed-classwise model must satisfy $g_u(x) + a_u > g_v(x) + a_v$, equivalently
\begin{equation}
a_u - a_v > g_v(x) - g_u(x) = t(x;\,u,v).
\end{equation}
Because the Bayes rule ranks $v$ above $u$ at $x'$, the same model must satisfy $g_v(x') + a_v > g_u(x') + a_u$, equivalently
\begin{equation}
a_u - a_v < g_v(x') - g_u(x') = t(x';\,u,v).
\end{equation}
Combining yields $t(x;\,u,v) < a_u - a_v < t(x';\,u,v)$, which requires $t(x;\,u,v) < t(x';\,u,v)$.
But by assumption $t(x;\,u,v) \geq t(x';\,u,v)$, a contradiction.
\end{proof}


\subsubsection{Attainable frontier under fixed retrieval}
\label{proof:frontier}

\begin{remark}
For any reranker $\hat{y}(X) \in S(X)$, the event $\{\hat{y}(X)=Y\}$ implies $C$, so $P(\hat{y}(X)=Y) = P(\hat{y}(X)=Y,C) \leq P(C)$.
Equivalently, $\mathrm{Hit@1}(\hat{y}) \leq 1$.
\end{remark}

\subsubsection{Conditional likelihood for the training objective}
\label{app:cond-likelihood}

For covered training examples $(x_i,S_i,y_i)$, the model
\[
q_{\theta,a}(y \mid x,S) = \frac{\exp\bigl(r_y(x,S)\bigr)}{\sum_{j \in S}\exp\bigl(r_j(x,S)\bigr)}
\]
is a standard multinomial logit model over the shortlist.
Assuming conditional independence across covered examples, the log-likelihood is $\sum_{i:\,y_i \in S_i} \log q_{\theta,a}(y_i \mid x_i,S_i)$, and adding $\ell_2$ penalties yields Eq.~(\ref{eq:training-obj}).

\subsubsection{Empirical-Bayes shrinkage}
\label{app:shrinkage}

A simple way to stabilize rare-class offsets is to treat the raw estimate $\tilde{a}_y$ as a noisy observation of an unobserved true offset $a_y$, and to place a Gaussian prior on $a_y$ within a frequency group $b(y)$:
\[
\tilde{a}_y \mid a_y \sim \mathcal{N}(a_y,\,\sigma^2_y), \qquad a_y \sim \mathcal{N}(\mu_{b(y)},\,\nu^2_{b(y)}).
\]
The posterior mean is
\begin{equation}
\label{eq:shrinkage-appendix}
\mathbb{E}[a_y \mid \tilde{a}_y] = (1-\lambda_y)\,\tilde{a}_y + \lambda_y\,\mu_{b(y)}, \qquad \lambda_y = \frac{\sigma^2_y}{\sigma^2_y + \nu^2_{b(y)}}.
\end{equation}
The more uncertain the class-specific estimate $\tilde{a}_y$ is (large $\sigma^2_y$), the more it is pulled toward the group mean $\mu_{b(y)}$.
For classes with many covered examples, $\sigma^2_y$ is small and the estimate stays close to the raw value; for rare classes with few examples, the estimate is regularized toward the group prior.

\section{Method Details}
\label{app:method}

\subsection{Algorithm pseudocode}
\label{app:algorithm}

\begin{algorithm}[h]
\caption{\repair{}: Reranking via Pairwise Residual Correction}
\label{alg:repair}
\begin{algorithmic}[1]
\Require Base scorer $g$, validation set $\{(x_i, y_i)\}_{i=1}^n$, shortlist size $k$
\State \textbf{Training:}
\For{$i = 1, \ldots, n$}
    \State $S_i \gets \text{top-}k$ classes by $g(x_i)$
\EndFor
\State $\mathcal{D} \gets \{(x_i, S_i, y_i) : y_i \in S_i\}$
\State \colorbox{red!8}{$(a^*, \theta^*) \gets \argmax_{a,\theta}\; \mathcal{L}(a, \theta;\, \mathcal{D})$} \hfill $\triangleright$ joint optimization
\For{each class $y$}
    \State $a^*_y \gets (1-\lambda_y)\,a^*_y + \lambda_y\,\mu_{b(y)}$ \hfill $\triangleright$ shrinkage
\EndFor
\State \textbf{Inference:}
\Require Input $x$, learned $(a^*, \theta^*)$
\State $S \gets \text{top-}k$ classes by $g(x)$
\For{$y \in S$}
    \State \colorbox{red!8}{$\ell_y \gets \frac{1}{k-1}\sum_{j \in S \setminus \{y\}} {\theta^*}^\top \phi(x, y, j)$} \hfill $\triangleright$ pairwise correction
    \State $r_y \gets \underbrace{g_y(x)}_{\text{base}} + \underbrace{a^*_y}_{\text{classwise}} + \underbrace{\ell_y}_{\text{pairwise}}$
\EndFor
\State \Return $\argmax_{y \in S}\; r_y$
\end{algorithmic}
\end{algorithm}

\subsection{Pairwise feature definitions}
\label{app:features}

Table~\ref{tab:features} gives the per-dataset pairwise feature specifications.

\begin{table}[h]
\centering
\caption{Pairwise feature vector $\phi(x,y,j)$ per dataset. All datasets share the first three competition features; the domain-specific similarity and log-frequency ratio vary.}
\label{tab:features}
\setlength{\tabcolsep}{4pt}
\renewcommand{\arraystretch}{1.1}
\resizebox{\columnwidth}{!}{\begin{tabular}{l ccccc c}
\toprule
& & & & & & \\[-6pt]
Feature & iNat & ImageNet-LT & Places-LT & GMDB & RareBench & Dim \\
\midrule
Score gap $g_y(x)-g_j(x)$ & \checkmark & \checkmark & \checkmark & \checkmark & \checkmark & 1 \\
Rank gap & \checkmark & \checkmark & \checkmark & \checkmark & \checkmark & 1 \\
Log-prob ratio $\log\frac{P(y|S)}{P(j|S)}$ & \checkmark & \checkmark & \checkmark & \checkmark & \checkmark & 1 \\
Domain similarity & Taxonomic & WordNet & --- & HPO Jaccard & HPO Jaccard & 0--1 \\
Log-freq ratio $\log\frac{n_y+1}{n_j+1}$ & \checkmark & \checkmark & \checkmark & \checkmark & \checkmark & 1 \\
\midrule
Total dimension & 5 & 5 & 4 & 5 & 5 & \\
\bottomrule
\end{tabular}}
\end{table}

The score gap directly encodes the base-score threshold $t(x;\,y,j) = g_j(x)-g_y(x)$ from Theorem~\ref{thm:pairwise-necessary}, allowing the pairwise correction to adapt to the current margin.
The domain similarity feature captures how confusable two classes are independent of the current input: taxonomic distance in the biological hierarchy for iNaturalist, WordNet path similarity for ImageNet-LT, and HPO phenotype Jaccard similarity for the rare-disease benchmarks.
Places-LT uses four features (no domain similarity available for scene categories).

\subsection{Implementation details}
\label{app:implementation}

\paragraph{Optimization.}
All reranker parameters are estimated using L-BFGS-B (\texttt{scipy.optimize.minimize}, maxiter=200--300, ftol=$10^{-11}$).
Wall-clock time is under 30 seconds per dataset per seed on a single CPU.

\paragraph{Hyperparameters.}
Table~\ref{tab:hyperparams} reports the per-dataset settings.
All values were selected by grid search on the calibration set (maximizing Hit@1 among covered examples); the test set was never used for selection.

\begin{table}[h]
\centering
\caption{Hyperparameter settings per dataset.}
\label{tab:hyperparams}
\setlength{\tabcolsep}{5pt}
\renewcommand{\arraystretch}{1.1}
\begin{tabular}{l ccccc}
\toprule
& iNat & ImageNet-LT & Places-LT & GMDB & RareBench \\
\midrule
$\lambda_a$ & 0.01 & 0.001 & 0.001 & 0.5 & 0.005 \\
$\lambda_\theta$ & 0.005 & 0.001 & 0.001 & 0.01 & 5$\times 10^{-5}$ \\
$\alpha$ & 0.3 & 0.0 & 0.3 & 0.3 & 0.0 \\
$\tau$ & 1.0 & 1.0 & 0.0 & 0.0 & 0.0 \\
Features & 5 & 5 & 4 & 5 & 5 \\
Shrinkage groups & 1 & 1 & 1 & 1 & 1 \\
\bottomrule
\end{tabular}
\end{table}

\paragraph{Hyperparameter search grid.}
$\tau \in \{0.0, 0.25, 0.5, 0.75, 1.0\}$;
$\lambda_a \in \{0.0005, 0.001, 0.005, 0.01, 0.05, 0.1, 0.5\}$;
$\lambda_\theta \in \{5{\times}10^{-5}, 10^{-4}, 5{\times}10^{-4}, 10^{-3}, 5{\times}10^{-3}, 10^{-2}\}$;
$\alpha \in \{0.0, 0.1, 0.3, 0.5, 0.7, 0.9\}$.

\paragraph{Base models.}
For iNaturalist, we fine-tune a ResNet-50 (torchvision \texttt{IMAGENET1K\_V2}) on the iNat~2018 training split.
For ImageNet-LT, we use the publicly available ERM ResNet-50 checkpoint from~\citet{kang2020decoupling}.\footnote{\url{https://github.com/facebookresearch/classifier-balancing}}
For Places-LT, we extract frozen ResNet-152 features (ImageNet-pretrained) and train a logistic-regression classifier with class-balanced weighting (cRT).
For GMDB,\footnote{\url{https://db.gestaltmatcher.org}} we select the 508 diseases (out of 528) that have facial phenotype HPO annotations and fine-tune Qwen3.5-0.8B,\footnote{\url{https://huggingface.co/Qwen/Qwen3.5-0.8B}} a multimodal vision-language model, as a 508-way classifier on 7{,}844 facial images; base logits are pre-normalized by classifier weight norms ($g_y \gets g_y / \|w_y\|^{0.5}$) before reranking.
For RareBench,\footnote{\url{https://github.com/chenxz1111/RareBench}} the same GMDB-trained checkpoint is evaluated in a text-only setting: HPO phenotype terms are formatted as a clinical prompt with a blank 224$\times$224 image placeholder, bypassing the visual modality entirely.

\section{Datasets and Evaluation}
\label{app:data}

\subsection{Dataset details}
\label{app:datasets}

\begin{table}[h]
\centering
\caption{Dataset overview.
$K$ = number of classes;
Cal/Test = calibration and test set sizes;
Rare/Freq = class counts under the uniform 80/20 split by training-set frequency.
All rerankers operate on a $k{=}10$ shortlist.}
\label{tab:datasets}
\setlength{\tabcolsep}{4pt}
\renewcommand{\arraystretch}{1.1}
\resizebox{\columnwidth}{!}{\begin{tabular}{l l r rr rr r l}
\toprule
Dataset & Base model & $K$ & Cal & Test & Rare & Freq & Recall@10 & Split source \\
\midrule
\textsc{iNaturalist} & ResNet-50 & 8{,}142 & 30{,}000 & 70{,}000 & 6{,}514 & 1{,}628 & 37.8\% & Original split \\
\textsc{ImageNet-LT} & ResNet-50 & 1{,}000 & 20{,}000 & 30{,}000 & 800 & 200 & 76.7\% & Random (seed=42) \\
\textsc{Places-LT} & ResNet-152 (cRT) & 365 & 10{,}804 & 7{,}446 & 292 & 73 & 75.3\% & Balanced val (seed=42) \\
\textsc{GMDB} & Qwen3.5-0.8B & 508 & 849 & 1{,}724 & 406 & 102 & 51.6\% & Original split \\
\textsc{RareBench} & Qwen3.5-0.8B (text) & 508 & 117 & 78 & --- & --- & 42.3\% & Random (seed=42) \\
\bottomrule
\end{tabular}}
\end{table}

For ImageNet-LT, we use the ERM ResNet-50 checkpoint from~\citet{kang2020decoupling} and split the 50{,}000-image ILSVRC~2012 validation set into 20{,}000 calibration and 30{,}000 test images (seed=42).
For Places-LT, we extract frozen ResNet-152 features and retrain a logistic-regression classifier with class-balanced weighting (cRT); the Places365 balanced validation set (${\sim}100$ images per class, 18{,}250 total) is split 60/40 into calibration and test (stratified, seed=42).
For RareBench, the GMDB-trained Qwen3.5-0.8B checkpoint is applied out-of-distribution to text-only clinical phenotype descriptions: HPO terms are formatted as a clinical prompt with a blank image placeholder, and the model outputs 508-dim logits over the GMDB disease label space.
Of 1{,}122 RareBench cases, 195 map to 67 GMDB diseases; the remaining cases involve diseases outside the GMDB label set.
The rare/frequent split for RareBench follows the GMDB training frequencies over the shared 508-class label space.

Figure~\ref{fig:longtail-dist} shows the training-set class frequency distribution for each benchmark, with the 80/20 rare/frequent boundary marked.
\begin{figure}[h]
\centering
\begin{subfigure}[t]{0.32\textwidth}
    \centering
    \includegraphics[width=\textwidth]{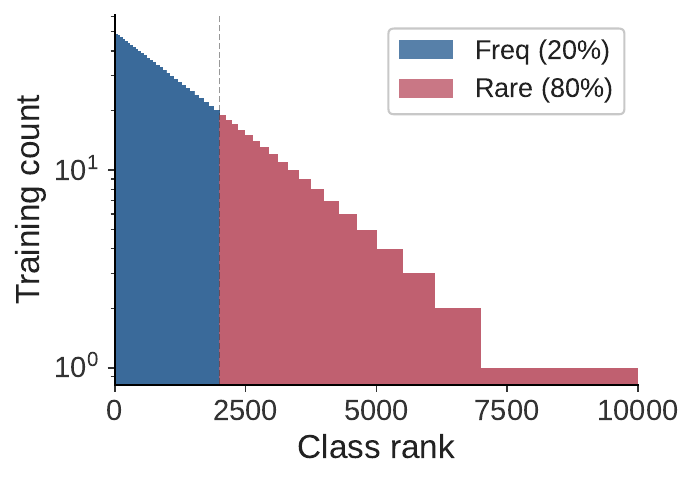}
    \caption{iNaturalist}
\end{subfigure}%
\hfill
\begin{subfigure}[t]{0.32\textwidth}
    \centering
    \includegraphics[width=\textwidth]{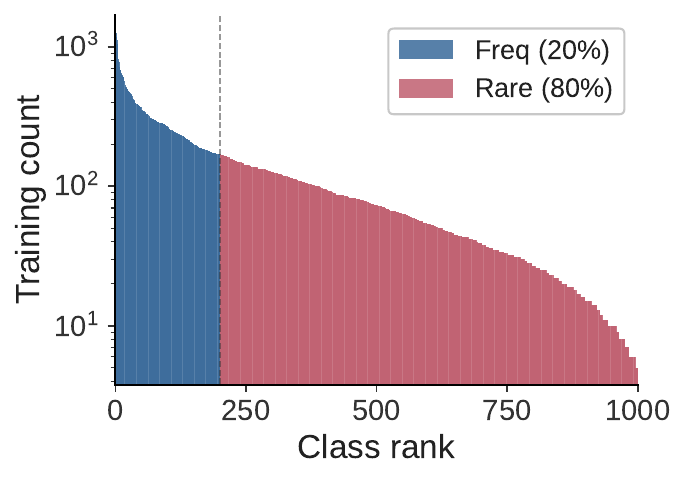}
    \caption{ImageNet-LT}
\end{subfigure}%
\hfill
\begin{subfigure}[t]{0.32\textwidth}
    \centering
    \includegraphics[width=\textwidth]{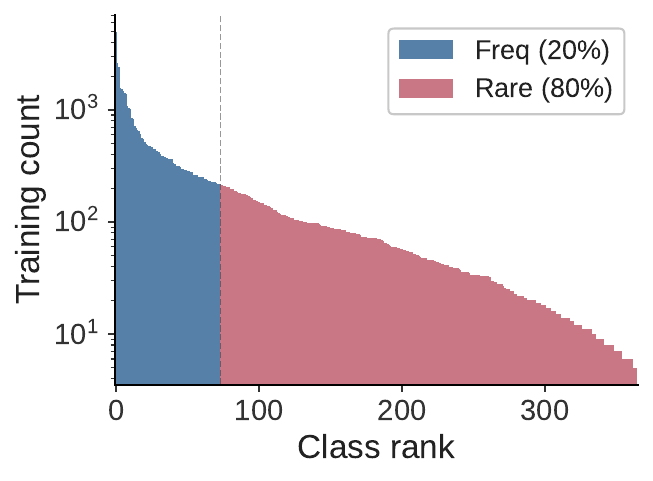}
    \caption{Places-LT}
\end{subfigure}

\vspace{6pt}
\begin{subfigure}[t]{0.32\textwidth}
    \centering
    \includegraphics[width=\textwidth]{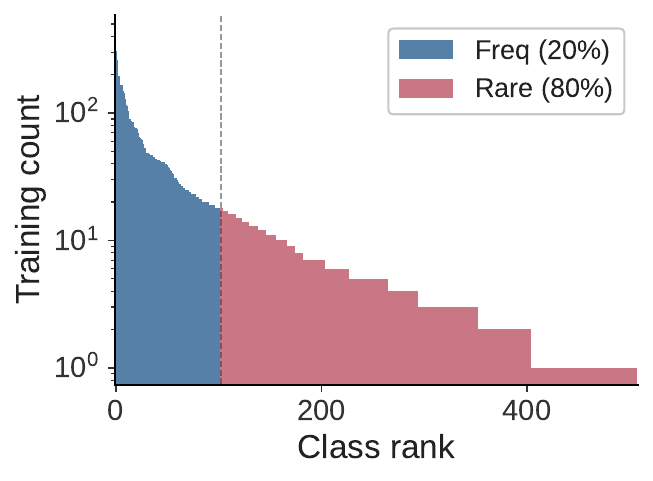}
    \caption{GMDB}
\end{subfigure}%
\hfill
\begin{subfigure}[t]{0.32\textwidth}
    \centering
    \includegraphics[width=\textwidth]{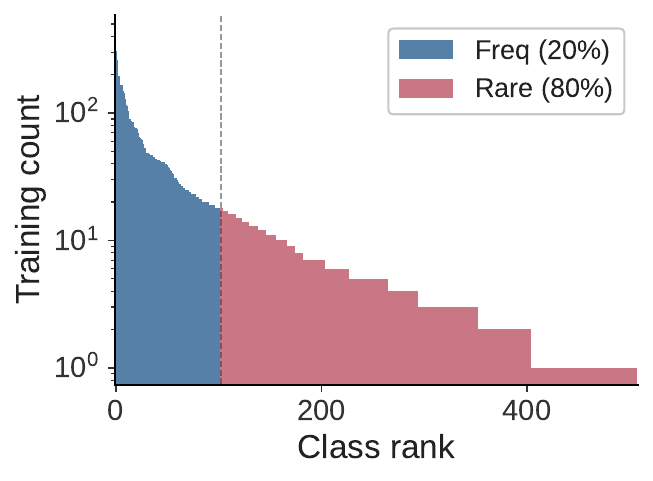}
    \caption{RareBench}
\end{subfigure}%
\hfill
\begin{subfigure}[t]{0.32\textwidth}
    \centering
    \includegraphics[width=\textwidth]{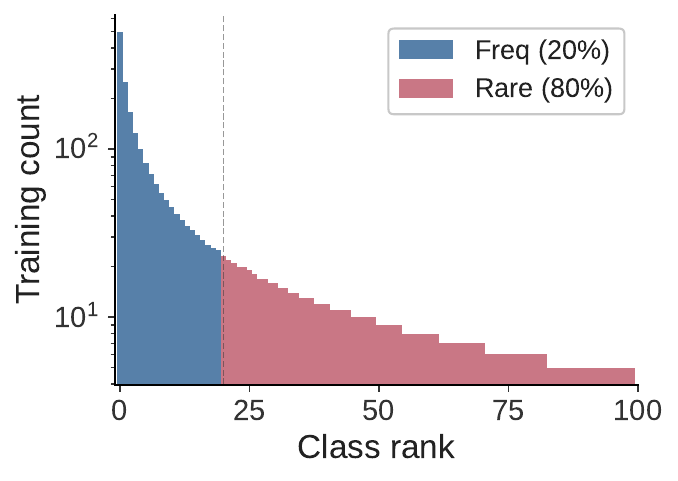}
    \caption{Synthetic}
\end{subfigure}
\caption{Training-set class frequency distributions.
Classes are sorted by frequency (descending); the vertical dashed line marks the 80/20 rare/frequent boundary.}
\label{fig:longtail-dist}
\end{figure}

\subsection{Evaluation metrics}
\label{app:metrics}

\paragraph{Hit@$k$.}
$\mathrm{Hit@}k(\hat{y}) = P(\mathrm{rank}_r(Y,X,S(X)) \leq k \mid C)$.
Hit@1 is coverage-conditioned top-1 accuracy within the reranked shortlist.
Rare Hit@1 and Frequent Hit@1 apply the same definition after restricting to the corresponding class subsets.

\paragraph{Mean Reciprocal Rank (MRR).}
$\mathrm{MRR}(\hat{y}) = \mathbb{E}[\mathrm{rank}_r(Y,X,S(X))^{-1} \mid C]$, where $\mathrm{rank}_r$ is the position of the true label in the reranked shortlist.

\paragraph{Recall@$k$.}
$\mathrm{Recall@}k = P(Y \in S(X))$, the shortlist coverage.
This is a property of the base model, not the reranker, and serves as the attainable ceiling for all reranking methods.

\paragraph{Recoverable gap closure (Eq.~\ref{eq:gap-closed}).}
$\rho_k = 1$ means the reranker closes the entire recoverable gap; $\rho_k = 0$ means no improvement.

\paragraph{Hardest-rival flip rate (Definition~\ref{def:hfr}).}
Conditions on examples where the base model errs, so the metric is not inflated by examples already handled correctly.

\section{Additional Experiments}
\label{app:experiments}

\subsection{Synthetic experiment details}
\label{sec:toy-details}

\paragraph{Data generation.}
We simulate $K=100$ classes with a Zipf prior $\pi_y \propto 1/y$, generating $n_{\mathrm{train}}=5{,}000$ training and $n_{\mathrm{test}}=2{,}000$ test examples.
Each class $y$ has a fixed mean $\mu_y \sim \mathcal{N}(0, 0.5I)$ in $\mathbb{R}^{10}$.
For each example, we sample $Y \sim \mathrm{Categorical}(\pi)$ and $x \sim \mathcal{N}(\mu_Y, I)$.
The base model score is $g_y(x) = \log p(y \mid x) + b_y + \epsilon$, where $b_y = -\tau\log\pi_y \cdot c$ is a class-level bias with corruption factor $c$, and $\epsilon \sim \mathcal{N}(0, 0.09)$.
The shortlist is $S(x)$ = top-10 classes by $g(x)$.

\paragraph{Class-separable regime.}
We set $c = 0.1$, introducing a strong class-level bias without any context-dependent confusion.
The residual $\beta^\star_y(x,S)$ is approximately constant per class up to a shortlist-shared constant.

\paragraph{Non-class-separable regime with contradictory pairs.}
We select 15 confuser pairs $(u,v)$ where $u$ is a tail class and $v$ is a head class.
For each example where both $u$ and $v$ appear in the shortlist, we perturb the base scores: $g_u(x) \mathrel{+}= s \cdot \delta$ and $g_v(x) \mathrel{-}= s \cdot \delta$, where $s \sim \mathrm{Uniform}\{-1,+1\}$ is a per-example context sign and $\delta = 3.0$ is the confusion strength.
This creates contradictory threshold requirements across contexts: $t(x;\,u,v)$ can be large positive in one context and large negative in another, which is exactly the condition of Theorem~\ref{thm:pairwise-necessary}.

\paragraph{Fitting.}
The pairwise feature vector $\phi(x,y,j)$ consists of three features: the base score gap $g_y(x)-g_j(x)$, the rank gap, and the cosine similarity $\cos(\mu_y,\mu_j)$.
We maximize the conditional log-likelihood~(\ref{eq:training-obj}) using L-BFGS with $\lambda_\theta = \lambda_a = 0.001$.
All methods share the same shortlist; only covered examples ($y_i \in S_i$) are used for fitting.
Results are averaged over 5 seeds.

\subsection{Shrinkage diagnostic}
\label{app:shrinkage-diagnostic}

For rare classes with only a handful of covered calibration examples, the unshrunk classwise estimate can have high variance.
Figure~\ref{fig:shrinkage} illustrates the intended behavior of the empirical-Bayes shrinkage in a controlled diagnostic: as calibration examples per class increase, the variance gap between the MLE offsets and the shrunk offsets narrows, while the tail-accuracy benefit of shrinkage is largest in the lowest-data regime.

\begin{figure}[h]
\centering
\includegraphics[width=0.55\textwidth]{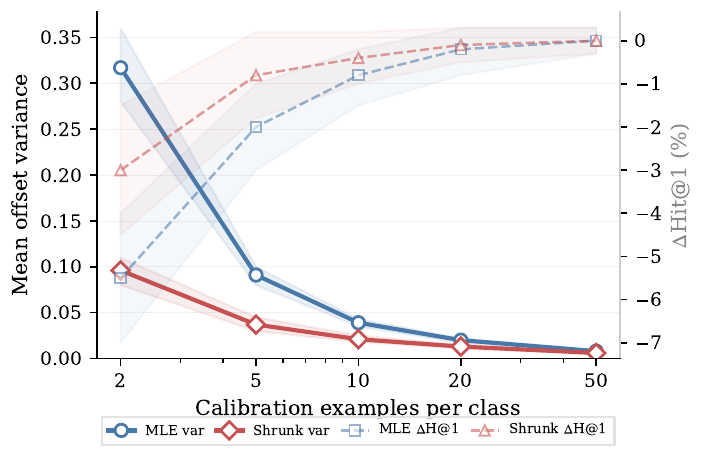}
\caption{Shrinkage diagnostic for the classwise correction.
Left axis: mean offset variance for the raw MLE and the shrunk estimate.
Right axis: $\Delta$Hit@1.
Shrinkage is most helpful when the number of covered calibration examples per class is very small.}
\label{fig:shrinkage}
\end{figure}

\subsection{Component ablation}
\label{app:ablation}

Figure~\ref{fig:ablation-synthetic} shows the synthetic ablation in both regimes. Figure~\ref{fig:ablation-real} extends the analysis to all five real datasets.

\begin{figure}[h]
\centering
\includegraphics[width=0.65\textwidth]{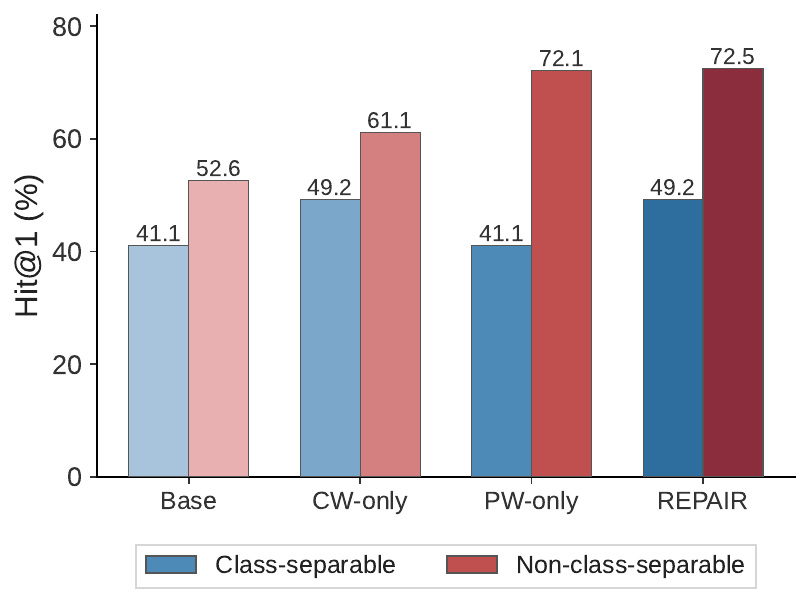}
\caption{Synthetic ablation ($K{=}100$, $k{=}10$).
CW-only: $\theta{=}0$ (classwise correction only).
PW-only: $a{=}0$ (pairwise correction only).
In the class-separable regime, CW-only matches REPAIR;
in the non-class-separable regime, both components contribute.}
\label{fig:ablation-synthetic}
\end{figure}

\begin{figure}[h]
\centering
\begin{subfigure}[t]{0.32\textwidth}
    \centering
    \includegraphics[width=\textwidth]{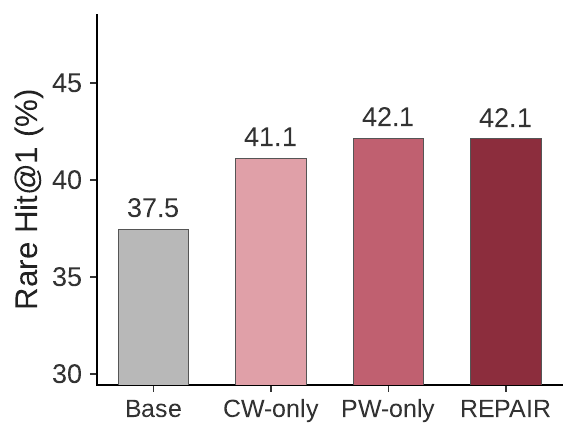}
    \caption{iNaturalist}
\end{subfigure}%
\hfill
\begin{subfigure}[t]{0.32\textwidth}
    \centering
    \includegraphics[width=\textwidth]{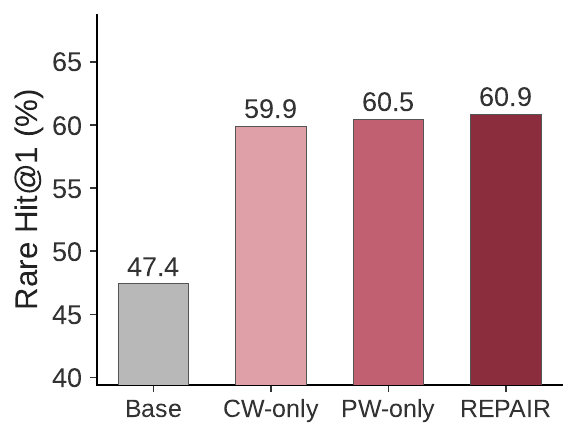}
    \caption{ImageNet-LT}
\end{subfigure}%
\hfill
\begin{subfigure}[t]{0.32\textwidth}
    \centering
    \includegraphics[width=\textwidth]{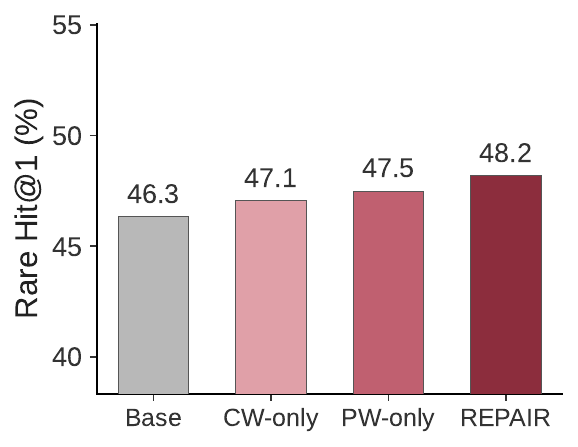}
    \caption{Places-LT}
\end{subfigure}

\vspace{6pt}
\begin{subfigure}[t]{0.32\textwidth}
    \centering
    \includegraphics[width=\textwidth]{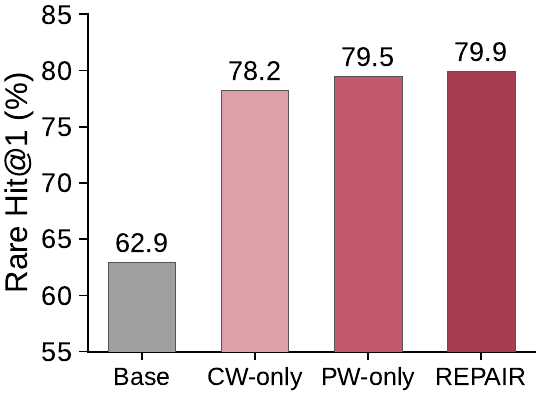}
    \caption{GMDB}
\end{subfigure}%
\hspace{0.02\textwidth}
\begin{subfigure}[t]{0.32\textwidth}
    \centering
    \includegraphics[width=\textwidth]{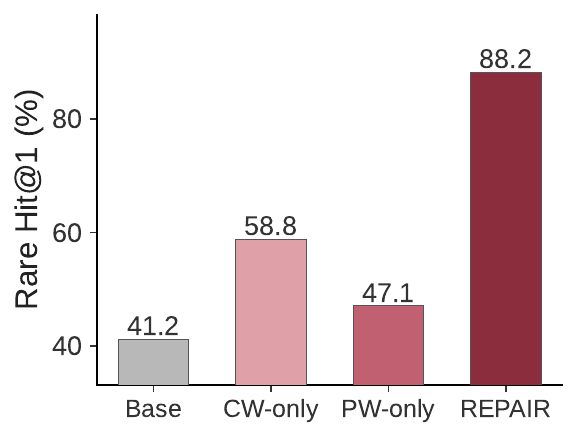}
    \caption{RareBench}
\end{subfigure}
\caption{Component ablation on real datasets (Rare Hit@1). CW-only: $\theta=0$. PW-only: $a=0$. On iNaturalist, PW-only matches \repair{} (42.1); on GMDB, PW-only closely matches \repair{} (79.5 vs.\ 79.9). On ImageNet-LT and Places-LT, both components contribute incrementally. On RareBench, neither component alone approaches the full model (CW-only 58.8, PW-only 47.1, \repair{} 88.2).}
\label{fig:ablation-real}
\end{figure}

\subsection{Real-data quintile analysis}
\label{app:quintile}

Figure~\ref{fig:quintile-real} shows the gain of REPAIR over Classwise stratified by mean per-class threshold dispersion $D_y$ on the four real datasets.

\begin{figure}[h]
\centering
\includegraphics[width=0.85\textwidth]{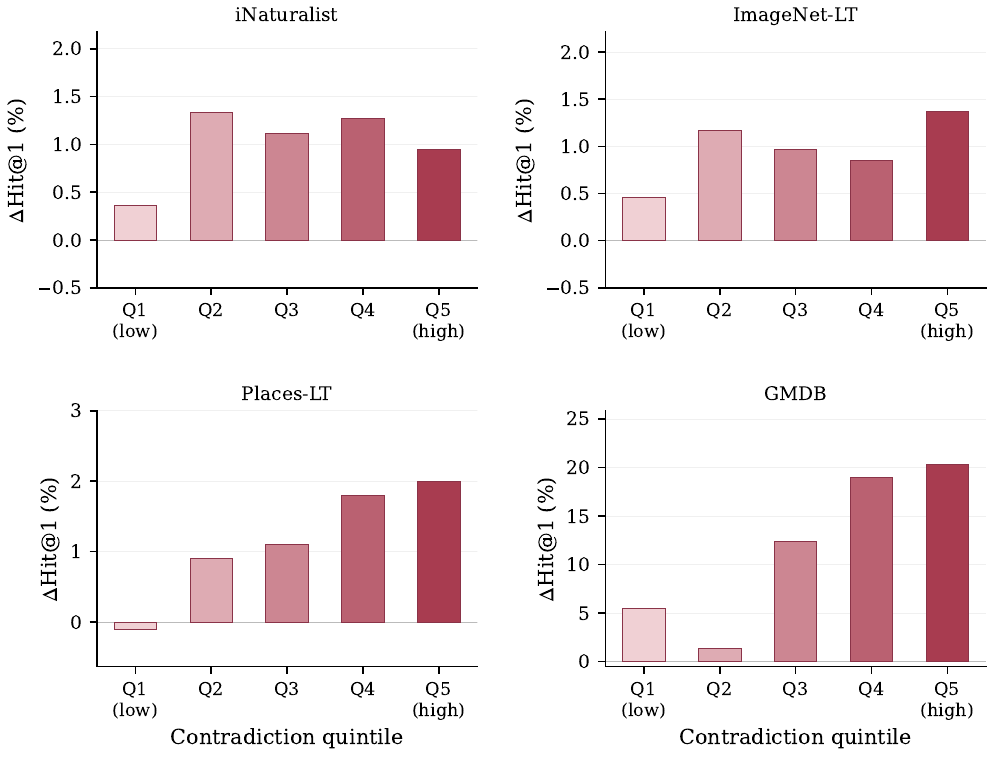}
\caption{$\Delta$Hit@1 of REPAIR over Classwise by mean $D_y$ quintile on real datasets.
Classes are sorted by their mean threshold dispersion and divided into five equal-sized bins.
Q5 (highest dispersion) shows the largest gain on ImageNet-LT, Places-LT, and GMDB.
On iNaturalist, the overall gain is small ($<1.5\%$ per bin) and the quintile gradient is flat, consistent with a near-class-separable residual.}
\label{fig:quintile-real}
\end{figure}

RareBench is omitted because its small number of covered classes yields fewer than 15 classes per quintile, making the stratified analysis unreliable.

\subsection{Sensitivity to shortlist size}
\label{app:k-sensitivity}

Table~\ref{tab:k_sensitivity} evaluates REPAIR across shortlist sizes $k \in \{5, 10, 20, 50\}$. REPAIR consistently outperforms all baselines at every $k$ on every dataset. As $k$ increases, coverage improves but the reranking task becomes harder; the relative advantage of REPAIR over Classwise is stable or grows with $k$.

\begin{table*}[h]
\centering
\caption{Sensitivity to shortlist size $k$. Hit@1 and Rare Hit@1 (\%) conditioned on coverage. Best per dataset-$k$ in $\mathbf{bold}$.}
\label{tab:k_sensitivity}
\vspace{2pt}
\setlength{\tabcolsep}{5pt}
\renewcommand{\arraystretch}{1.08}
\resizebox{\textwidth}{!}{\begin{tabular}{ll cccc cccc}
\toprule
& & \multicolumn{4}{c}{Hit@1 $\uparrow$} & \multicolumn{4}{c}{Rare Hit@1 $\uparrow$} \\
\cmidrule(lr){3-6} \cmidrule(lr){7-10}
Dataset & Method & $k{=}5$ & $k{=}10$ & $k{=}20$ & $k{=}50$ & $k{=}5$ & $k{=}10$ & $k{=}20$ & $k{=}50$ \\
\midrule
Synthetic & Base & $52.6$ & $45.8$ & $42.6$ & $41.2$ & -- & -- & -- & -- \\
(class-sep.) & Classwise & $63.0$ & $54.8$ & $50.3$ & $48.1$ & -- & -- & -- & -- \\
 & REPAIR & $\mathbf{63.0}$ & $\mathbf{54.8}$ & $\mathbf{50.3}$ & $\mathbf{48.1}$ & -- & -- & -- & -- \\
\midrule
Synthetic & Base & $59.0$ & $52.6$ & $49.6$ & $48.4$ & -- & -- & -- & -- \\
(non-class-sep.) & Classwise & $68.1$ & $61.1$ & $57.2$ & $55.2$ & -- & -- & -- & -- \\
 & REPAIR & $\mathbf{79.5}$ & $\mathbf{72.5}$ & $\mathbf{68.3}$ & $\mathbf{66.1}$ & -- & -- & -- & -- \\
\midrule
iNaturalist & Base & $56.1$ & $47.0$ & $40.5$ & $34.5$ & $47.6$ & $37.5$ & $30.6$ & $24.6$ \\
 & LogitAdj & $58.2$ & $49.0$ & $42.3$ & $36.1$ & $52.4$ & $41.7$ & $34.4$ & $27.7$ \\
 & Classwise & $58.1$ & $48.8$ & $42.1$ & $35.9$ & $51.8$ & $41.1$ & $33.7$ & $27.1$ \\
 & REPAIR & $\mathbf{58.4}$ & $\mathbf{49.0}$ & $\mathbf{42.3}$ & $\mathbf{35.9}$ & $\mathbf{52.9}$ & $\mathbf{42.2}$ & $\mathbf{34.7}$ & $\mathbf{27.9}$ \\
\midrule
ImageNet-LT & Base & $61.6$ & $54.4$ & $49.5$ & $45.7$ & $55.2$ & $47.4$ & $42.4$ & $38.5$ \\
 & LogitAdj & $69.2$ & $61.5$ & $56.0$ & $51.7$ & $68.4$ & $59.5$ & $53.2$ & $48.4$ \\
 & Classwise & $69.5$ & $61.8$ & $56.3$ & $51.9$ & $68.9$ & $59.9$ & $53.6$ & $48.7$ \\
 & REPAIR & $\mathbf{69.7}$ & $\mathbf{62.0}$ & $\mathbf{56.4}$ & $\mathbf{52.0}$ & $\mathbf{69.7}$ & $\mathbf{60.8}$ & $\mathbf{54.4}$ & $\mathbf{49.3}$ \\
\midrule
Places-LT & Base & $52.4$ & $44.7$ & $39.5$ & $35.8$ & $54.0$ & $46.3$ & $41.0$ & $37.2$ \\
 & LogitAdj & $52.1$ & $44.5$ & $39.3$ & $35.6$ & $55.8$ & $47.9$ & $42.3$ & $38.5$ \\
 & Classwise & $52.9$ & $45.1$ & $39.9$ & $36.0$ & $54.8$ & $46.9$ & $41.6$ & $37.6$ \\
 & REPAIR & $\mathbf{53.4}$ & $\mathbf{45.5}$ & $\mathbf{39.9}$ & $\mathbf{36.2}$ & $\mathbf{56.2}$ & $\mathbf{48.2}$ & $\mathbf{42.4}$ & $\mathbf{38.5}$ \\
\midrule
GMDB & Base & $75.2$ & $64.0$ & $56.0$ & $45.8$ & $70.5$ & $62.9$ & $53.5$ & $41.1$ \\
 & LogitAdj & $78.7$ & $70.6$ & $58.9$ & $48.1$ & $77.8$ & $78.8$ & $59.6$ & $45.8$ \\
 & Classwise & $78.9$ & $71.4$ & $59.3$ & $48.4$ & $79.2$ & $78.5$ & $60.1$ & $46.9$ \\
 & REPAIR & $\mathbf{79.5}$ & $\mathbf{72.3}$ & $\mathbf{60.2}$ & $\mathbf{49.0}$ & $\mathbf{81.5}$ & $\mathbf{79.6}$ & $\mathbf{63.8}$ & $\mathbf{48.5}$ \\
\midrule
RareBench & Base & $62.1$ & $59.4$ & $50.0$ & $37.5$ & $46.7$ & $41.2$ & $38.9$ & $25.9$ \\
 & LogitAdj & $65.5$ & $62.5$ & $52.8$ & $39.6$ & $53.3$ & $52.9$ & $44.4$ & $29.6$ \\
 & Classwise & $72.4$ & $66.2$ & $72.2$ & $56.2$ & $60.0$ & $54.1$ & $77.8$ & $55.6$ \\
 & REPAIR & $\mathbf{86.2}$ & $\mathbf{85.6}$ & $\mathbf{88.9}$ & $\mathbf{72.9}$ & $\mathbf{80.0}$ & $\mathbf{81.2}$ & $\mathbf{88.9}$ & $\mathbf{70.4}$ \\
\bottomrule
\end{tabular}}
\end{table*}

\subsection{Unconditional results}
\label{app:unconditional}

Table~\ref{tab:unconditional} reports unconditional Hit@1, which factors in coverage: unconditional Hit@1 = conditional Hit@1 $\times$ Recall@$k$ / 100. Since all methods share the same shortlist, unconditional differences mirror the conditional ones.


\begin{table}[h]
\centering
\caption{Unconditional results ($k{=}10$). Unconditional Hit@1 equals conditional Hit@1 scaled by Recall@10. Best per dataset in $\mathbf{bold}$.}
\label{tab:unconditional}
\vspace{2pt}
\setlength{\tabcolsep}{5pt}
\renewcommand{\arraystretch}{1.08}
\begin{tabular}{ll c cc cc}
\toprule
& & & \multicolumn{2}{c}{Hit@1} & \multicolumn{2}{c}{Rare Hit@1} \\
\cmidrule(lr){4-5} \cmidrule(lr){6-7}
Dataset & Method & Recall & Cond. & Uncond. & Cond. & Uncond. \\
\midrule
iNaturalist & Base & $37.8$ & $47.0$ & $17.8$ & $37.5$ & $14.2$ \\
 & LogitAdj & & $49.0$ & $18.5$ & $41.7$ & $15.8$ \\
 & Classwise & & $48.8$ & $18.5$ & $41.1$ & $15.6$ \\
 & REPAIR & & $\mathbf{49.0}$ & $\mathbf{18.5}$ & $\mathbf{42.2}$ & $\mathbf{15.9}$ \\
\midrule
ImageNet-LT & Base & $76.7$ & $54.4$ & $41.7$ & $47.4$ & $36.4$ \\
 & LogitAdj & & $61.5$ & $47.2$ & $59.5$ & $45.6$ \\
 & Classwise & & $61.8$ & $47.4$ & $59.9$ & $45.9$ \\
 & REPAIR & & $\mathbf{62.0}$ & $\mathbf{47.6}$ & $\mathbf{60.8}$ & $\mathbf{46.6}$ \\
\midrule
Places-LT & Base & $75.3$ & $44.7$ & $33.7$ & $46.3$ & $34.9$ \\
 & LogitAdj & & $44.5$ & $33.5$ & $47.9$ & $36.1$ \\
 & Classwise & & $45.1$ & $34.0$ & $46.9$ & $35.3$ \\
 & REPAIR & & $\mathbf{45.5}$ & $\mathbf{34.3}$ & $\mathbf{48.2}$ & $\mathbf{36.3}$ \\
\midrule
GMDB & Base & $51.6$ & $64.0$ & $33.0$ & $62.9$ & $32.5$ \\
 & LogitAdj & & $70.6$ & $36.4$ & $78.8$ & $40.7$ \\
 & Classwise & & $71.4$ & $36.8$ & $78.5$ & $40.5$ \\
 & REPAIR & & $\mathbf{72.3}$ & $\mathbf{37.3}$ & $\mathbf{79.6}$ & $\mathbf{41.1}$ \\
\midrule
RareBench & Base & $42.3$ & $59.4$ & $25.1$ & $41.2$ & $17.4$ \\
 & LogitAdj & & $62.5$ & $26.4$ & $52.9$ & $22.4$ \\
 & Classwise & & $66.2$ & $28.0$ & $54.1$ & $22.9$ \\
 & REPAIR & & $\mathbf{85.6}$ & $\mathbf{36.2}$ & $\mathbf{81.2}$ & $\mathbf{34.3}$ \\
\bottomrule
\end{tabular}
\end{table}
\subsection{Sign test results}
\label{app:signtest}

Table~\ref{tab:signtest} reports the sign test results for REPAIR vs.\ Classwise across all dataset--metric pairs.
Each cell shows the number of trials (out of 5) in which REPAIR strictly outperforms Classwise.

\begin{table}[h]
\centering
\caption{Sign test results (REPAIR vs.\ Classwise, 5 calibration subsamples).
$\geq$4/5 wins = significant ($\checkmark$).}
\label{tab:signtest}
\setlength{\tabcolsep}{5pt}
\renewcommand{\arraystretch}{1.1}
\begin{tabular}{l ccccc}
\toprule
Metric & iNat & ImageNet-LT & Places-LT & GMDB & RareBench \\
\midrule
Hit@1 & 5/5 \checkmark & 5/5 \checkmark & 5/5 \checkmark & 5/5 \checkmark & 4/5 \checkmark \\
Hit@3 & 5/5 \checkmark & 4/5 \checkmark & 5/5 \checkmark & 5/5 \checkmark & 5/5 \checkmark \\
MRR & 5/5 \checkmark & 4/5 \checkmark & 5/5 \checkmark & 5/5 \checkmark & 5/5 \checkmark \\
Rare Hit@1 & 5/5 \checkmark & 5/5 \checkmark & 5/5 \checkmark & 5/5 \checkmark & 4/5 \checkmark \\
Freq Hit@1 & 0/5 & 0/5 & 0/5 & 0/5 & 5/5 \checkmark \\
HFR & 5/5 \checkmark & 5/5 \checkmark & 5/5 \checkmark & 5/5 \checkmark & 3/5 \\
\bottomrule
\end{tabular}
\end{table}

REPAIR significantly outperforms Classwise on Hit@1, Hit@3, MRR, and Rare Hit@1 across all five datasets (20/20 pass at $\geq$4/5).
HFR passes on four of five datasets; RareBench (3/5) is limited by its small covered test set (32 examples).
Frequent Hit@1 passes only on RareBench, where pairwise competition benefits both strata; on the remaining datasets the pairwise term introduces slight noise on frequent classes, as discussed in \S\ref{sec:results}.

\subsection{RareBench bootstrap confidence intervals}
\label{app:rarebench-ci}

Because RareBench has only 32 covered test examples, we report bootstrap 95\% confidence intervals (10{,}000 resamples) to assess the reliability of the observed gains.

\begin{table}[h]
\centering
\caption{RareBench bootstrap 95\% CI ($n{=}32$, 10K resamples). Even at the lower bound of the 95\% CI, REPAIR outperforms Classwise by $+6.2$\,\%.}
\label{tab:rarebench-ci}
\vspace{2pt}
\setlength{\tabcolsep}{6pt}
\renewcommand{\arraystretch}{1.08}
\begin{tabular}{l c c}
\toprule
Method & Hit@1 (\%) & 95\% CI \\
\midrule
Base & $59.4$ & $[40.6,\; 75.0]$ \\
Classwise & $66.2$ & $[50.5,\; 81.8]$ \\
REPAIR & $\mathbf{85.6}$ & $[73.1,\; 95.0]$ \\
\midrule
REPAIR $-$ Classwise & $+19.4$ & $[+6.2,\; +31.2]$ \\
\bottomrule
\end{tabular}
\end{table}

The probability that REPAIR outperforms Classwise under resampling is $P = 0.999$. On rare classes specifically, the gap is $+27.1\%$ with $95\%$ CI $[+11.8, +52.9]$ ($P = 0.998$). Despite the small sample size, the improvement is statistically robust.


\subsection{Pure-text NLP benchmarks}
\label{app:nlp}

To test whether the residual decomposition holds beyond vision and clinical
settings, we evaluate on four pure-text NLP benchmarks: \textsc{GoEmotions}
(emotion classification, 28 classes), \textsc{MASSIVE} (intent classification,
60 classes), \textsc{HuffPost-LT} (news-topic classification, 42 classes), and
\textsc{Few-NERD} (entity-type classification, 66 classes). All four use a
Qwen3.5-2B base classifier fine-tuned with LoRA, and follow the same
calibration/test protocol as the main experiments: reranker parameters are
estimated on a held-out calibration split and evaluated on a disjoint test split,
hyperparameters are selected on the calibration split by maximizing Hit@1 among
covered examples, and all metrics are conditioned on coverage ($Y \in S$). We
compare Base, LogitAdj, $\tau$-norm, Classwise, and REPAIR under the identical
protocol, with Classwise and REPAIR reporting mean$\pm$std over 5 calibration
subsamples.

Table~\ref{tab:main_nlp} reports the full comparison. The class-separable /
non-class-separable distinction studied in the main text reproduces in pure-text
LM classifiers. On \textsc{GoEmotions}, the residual is class-separable: Classwise
already captures nearly all of the recoverable correction, and REPAIR adds only
$+0.05$ Hit@1 and $+0.51$ Rare Hit@1 over it. On \textsc{MASSIVE},
\textsc{HuffPost-LT}, and \textsc{Few-NERD}, the residual is non-class-separable,
and the pairwise term yields gains a fixed classwise offset cannot: $+1.78$,
$+0.84$, and $+1.05$ Hit@1 over Classwise, with Rare Hit@1 gains of $+5.53$,
$+2.04$, and $+1.83$, respectively. On \textsc{Few-NERD}, LogitAdj improves Rare
Hit@1 only modestly over Base ($18.60 \to 19.90$), Classwise gives a larger
improvement ($\to 29.75$), and REPAIR improves it further ($\to 31.58$),
indicating that a learned classwise correction helps but the pairwise term still
captures additional label-competition structure. This is consistent with the
framework's prediction: pairwise correction helps precisely where the classwise
component is insufficient.

\begin{table*}[t]
\centering
\caption{Reranking results on four pure-text NLP benchmarks (shortlist size 10),
all using a Qwen3.5-2B base classifier fine-tuned with LoRA. All metrics are
conditioned on the true label appearing in the shortlist; arrows indicate higher
is better. Rare and Frequent denote the bottom $80\%$ and top $20\%$ of classes
by training-set frequency. Both Classwise and \repair{} report
mean{\scriptsize$\pm$std} over 5 calibration subsamples. Best per dataset in
\textbf{bold}. Row shading: \colorbox{blue!6}{blue} = class-separable,
\colorbox{red!8}{red} = non-class-separable.}
\label{tab:main_nlp}
\vspace{2pt}
\setlength{\tabcolsep}{5pt}
\renewcommand{\arraystretch}{1.10}
\resizebox{\textwidth}{!}{\begin{tabular}{l l cc cc c}
\toprule
& & \multicolumn{2}{c}{\textit{Overall}} & \multicolumn{2}{c}{\textit{Stratified Hit@1}} & \\
\cmidrule(lr){3-4} \cmidrule(lr){5-6}
Dataset & Method & Hit@1$\uparrow$ & Hit@3$\uparrow$ & Rare$\uparrow$ & Freq.$\uparrow$ & HFR$\uparrow$ \\
\midrule
\multirow{5}{*}{\shortstack[l]{\textsc{GoEmotions}\\[-1pt]{\scriptsize(28 cls)}}}
 & Base        & $58.10$ & $81.23$ & $26.90$ & $74.97$ & --    \\
 & LogitAdj    & $58.48${\tiny$\pm.05$} & $\mathbf{81.58}${\tiny$\pm.16$} & $29.97${\tiny$\pm.30$} & $73.89${\tiny$\pm.12$} & $.042$ \\
 & $\tau$-norm & $58.27${\tiny$\pm.00$} & $81.07${\tiny$\pm.00$} & $27.21${\tiny$\pm.00$} & $\mathbf{75.08}${\tiny$\pm.00$} & $.017$ \\
 & Classwise   & $58.82${\tiny$\pm.07$} & $81.54${\tiny$\pm.13$} & $34.94${\tiny$\pm.81$} & $71.93${\tiny$\pm.42$} & $.093$ \\
 & \cellcolor{blue!10} REPAIR (Ours)
   & \cellcolor{blue!10} $\mathbf{58.87}${\tiny$\pm.39$}
   & \cellcolor{blue!10} $81.53${\tiny$\pm.07$}
   & \cellcolor{blue!10} $\mathbf{35.45}${\tiny$\pm2.95$}
   & \cellcolor{blue!10} $71.92${\tiny$\pm.89$}
   & \cellcolor{blue!10} $\mathbf{.139}$ \\
\midrule
\multirow{5}{*}{\shortstack[l]{\textsc{MASSIVE}\\[-1pt]{\scriptsize(60 cls)}}}
 & Base        & $55.26$ & $77.65$ & $32.06$ & $68.55$ & --    \\
 & LogitAdj    & $56.39${\tiny$\pm.00$} & $78.86${\tiny$\pm.00$} & $40.07${\tiny$\pm.00$} & $68.55${\tiny$\pm.00$} & $.133$ \\
 & $\tau$-norm & $55.26${\tiny$\pm.00$} & $77.65${\tiny$\pm.00$} & $32.06${\tiny$\pm.00$} & $68.55${\tiny$\pm.00$} & $.000$ \\
 & Classwise   & $56.65${\tiny$\pm.16$} & $78.71${\tiny$\pm.33$} & $40.26${\tiny$\pm.55$} & $67.70${\tiny$\pm.19$} & $.122$ \\
 & \cellcolor{red!10} REPAIR (Ours)
   & \cellcolor{red!10} $\mathbf{58.43}${\tiny$\pm.26$}
   & \cellcolor{red!10} $\mathbf{80.79}${\tiny$\pm.20$}
   & \cellcolor{red!10} $\mathbf{45.79}${\tiny$\pm.62$}
   & \cellcolor{red!10} $\mathbf{71.30}${\tiny$\pm.61$}
   & \cellcolor{red!10} $\mathbf{.217}$ \\
\midrule
\multirow{5}{*}{\shortstack[l]{\textsc{HuffPost-LT}\\[-1pt]{\scriptsize(42 cls)}}}
 & Base        & $66.51$ & $86.87$ & $41.15$ & $\mathbf{81.90}$ & --    \\
 & LogitAdj    & $66.62${\tiny$\pm.04$} & $87.19${\tiny$\pm.06$} & $42.61${\tiny$\pm.11$} & $81.20${\tiny$\pm.08$} & $.026$ \\
 & $\tau$-norm & $66.47${\tiny$\pm.08$} & $86.82${\tiny$\pm.10$} & $41.12${\tiny$\pm.05$} & $81.85${\tiny$\pm.10$} & $.002$ \\
 & Classwise   & $66.14${\tiny$\pm.15$} & $86.58${\tiny$\pm.14$} & $43.39${\tiny$\pm1.69$} & $80.60${\tiny$\pm.84$} & $.060$ \\
 & \cellcolor{red!10} REPAIR (Ours)
   & \cellcolor{red!10} $\mathbf{66.98}${\tiny$\pm.19$}
   & \cellcolor{red!10} $\mathbf{87.53}${\tiny$\pm.18$}
   & \cellcolor{red!10} $\mathbf{45.43}${\tiny$\pm2.33$}
   & \cellcolor{red!10} $80.07${\tiny$\pm1.14$}
   & \cellcolor{red!10} $\mathbf{.062}$ \\
\midrule
\multirow{5}{*}{\shortstack[l]{\textsc{Few-NERD}\\[-1pt]{\scriptsize(66 cls)}}}
 & Base        & $37.13$ & $65.48$ & $18.60$ & $39.63$ & --    \\
 & LogitAdj    & $38.14${\tiny$\pm.00$} & $67.34${\tiny$\pm.00$} & $19.90${\tiny$\pm.00$} & $40.43${\tiny$\pm.00$} & $.132$ \\
 & $\tau$-norm & $37.06${\tiny$\pm.03$} & $65.30${\tiny$\pm.09$} & $18.42${\tiny$\pm.09$} & $39.57${\tiny$\pm.03$} & $.001$ \\
 & Classwise   & $38.64${\tiny$\pm.03$} & $68.82${\tiny$\pm.14$} & $29.75${\tiny$\pm.83$} & $39.84${\tiny$\pm.09$} & $.050$ \\
 & \cellcolor{red!10} REPAIR (Ours)
   & \cellcolor{red!10} $\mathbf{39.69}${\tiny$\pm.03$}
   & \cellcolor{red!10} $\mathbf{70.11}${\tiny$\pm.21$}
   & \cellcolor{red!10} $\mathbf{31.58}${\tiny$\pm1.08$}
   & \cellcolor{red!10} $\mathbf{40.78}${\tiny$\pm.13$}
   & \cellcolor{red!10} $\mathbf{.137}$ \\
\bottomrule
\end{tabular}}
\end{table*}

\subsection{HPO ablation on GMDB}
\label{app:hpo-ablation}

A natural concern is whether REPAIR's gains in the rare-disease setting come from
the HPO phenotype-similarity feature, an external source of medical knowledge,
rather than from the pairwise correction itself. To test this, we remove the HPO
Jaccard feature from $\phi(x,y,j)$ on GMDB while keeping the same calibration
protocol and five calibration seeds. Table~\ref{tab:hpo-ablation} reports the
result. Removing HPO changes every REPAIR metric by less than $0.1$, well within
seed standard deviation, so the aggregate gain is not driven by HPO specifically:
the model-derived competition features (score gap, rank gap, log-probability
ratio, log-frequency ratio) carry most of it. HPO remains useful qualitatively by
shaping the composition of the top-$k$ differential diagnosis list, as illustrated
in Appendix~\ref{app:casestudy}.

\begin{table}[h]
\centering
\caption{GMDB ablation removing the HPO Jaccard feature (mean$\pm$std over 5
calibration subsamples). This ablation uses a separate calibration split, so
absolute Hit@1 differs slightly from Table~\ref{tab:main_combined}; the relevant
comparison is REPAIR with HPO versus without HPO under the same split.}
\label{tab:hpo-ablation}
\vspace{2pt}
\setlength{\tabcolsep}{8pt}
\renewcommand{\arraystretch}{1.10}
\begin{tabular}{l cccc}
\toprule
Method & Hit@1$\uparrow$ & Rare$\uparrow$ & Freq.$\uparrow$ & HFR$\uparrow$ \\
\midrule
Classwise            & $67.25${\tiny$\pm.68$} & $66.61${\tiny$\pm.74$} & $67.69${\tiny$\pm.92$} & $.301$ \\
REPAIR with HPO      & $70.55${\tiny$\pm.57$} & $76.12${\tiny$\pm.39$} & $66.78${\tiny$\pm1.11$} & $.466$ \\
REPAIR without HPO   & $70.50${\tiny$\pm.53$} & $76.06${\tiny$\pm.43$} & $66.75${\tiny$\pm1.07$} & $.466$ \\
\bottomrule
\end{tabular}
\end{table}

\subsection{GMDB case study: the qualitative role of HPO}
\label{app:casestudy}

The HPO ablation in Appendix~\ref{app:hpo-ablation} shows that HPO is not the main
source of REPAIR's aggregate Hit@1 gain. Its contribution is instead qualitative:
it shapes the composition of the top-$k$ differential diagnosis list. We
illustrate this with one held-out GMDB patient.

The patient presents with long palpebral fissures with eversion of the lateral
lower eyelid, arched eyebrows, large protruding earlobes, mild intellectual
disability, congenital heart defect, and short stature. The correct diagnosis is
Kabuki syndrome. In the base Qwen3.5-0.8B shortlist, the clinically related CHARGE
syndrome appears at rank~5 while the correct Kabuki syndrome appears at rank~8, so
the case is recoverable but initially misranked. The two syndromes share
overlapping facial, cardiac, hearing-related, and developmental phenotypes and are
often part of the same clinical differential.

Table~\ref{tab:casestudy} shows the selected ranks after reranking. A fixed
frequency-based offset (LogitAdj) can lift Kabuki syndrome to rank~1, but it
pushes CHARGE syndrome down to rank~7: it corrects the top-1 label while
disrupting a clinically meaningful differential. REPAIR with HPO similarity also
promotes Kabuki syndrome to rank~1, but keeps CHARGE syndrome at rank~2. Because
Kabuki and CHARGE share relevant phenotype terms, the HPO-aware pairwise term
gives their pair a stronger positive correction than unrelated rare diseases
receive. Thus the pairwise HPO term does not merely improve top-1 accuracy; it
changes the structure of the top-$k$ list, lifting the correct diagnosis while
preserving a phenotypically related alternative near the top.

\begin{table}[h]
\centering
\caption{Selected ranks for one held-out GMDB patient (correct diagnosis: Kabuki
syndrome). CHARGE syndrome is a clinically related differential. LogitAdj corrects
the top-1 prediction but pushes the related diagnosis down the list, whereas
REPAIR with HPO similarity corrects the top-1 prediction while keeping the related
diagnosis near the top.}
\label{tab:casestudy}
\vspace{2pt}
\setlength{\tabcolsep}{8pt}
\renewcommand{\arraystretch}{1.15}
\begin{tabular}{l cc}
\toprule
Method & Kabuki (correct) & CHARGE (related) \\
\midrule
Base Qwen3.5-0.8B          & Rank 8 & Rank 5 \\
LogitAdj                   & Rank 1 & Rank 7 \\
REPAIR with HPO similarity & Rank 1 & Rank 2 \\
\bottomrule
\end{tabular}
\end{table}
\section{Limitations}
\label{app:limitations}

\paragraph{Domain-specific features.}
The pairwise feature vector includes a domain-specific similarity component (taxonomic distance, WordNet similarity, or HPO Jaccard) on four of five datasets. When such structured knowledge is unavailable, as on Places-LT, we fall back to a four-dimensional feature vector without domain similarity and still observe gains. A natural alternative in the absence of external knowledge is cosine similarity between the base model's class embeddings, which we leave for future work.

\paragraph{Shortlist coverage ceiling.}
REPAIR operates on a fixed top-$k$ shortlist produced by the base model and cannot recover examples where the true label falls outside the shortlist. This is a fundamental constraint shared by all post-hoc rerankers, not specific to our method. Our framework explicitly conditions on coverage ($Y \in S$) and focuses on maximizing ranking quality within this recoverable set. Improving coverage itself requires modifying the base model, which is orthogonal to the reranking problem we study.